\documentclass[twoside]{article}

\usepackage{hyperref}

\usepackage[numbers, compress]{natbib}
\usepackage[acronym]{glossaries}

\usepackage{amsfonts}       
\usepackage{amsmath}
\usepackage{amssymb}
\usepackage{amsthm}
\usepackage{booktabs}       
\usepackage{epsfig}
\usepackage{graphicx}
\usepackage{makecell}
\usepackage{microtype}      
\usepackage{multirow}
\usepackage{nicefrac}       
\usepackage{subcaption}
\usepackage{times}
\usepackage[marginpar]{todo}
\usepackage[binary-units=true]{siunitx}

\DeclareSIUnit{\nothing}{\relax}

\newtheorem{procedure}{Procedure}

\newtheorem{lemma}{Lemma}

\newtheorem{definition}{Definition}
\newtheorem{properties}{Property}
\newtheorem{algorithm}{Algorithm}

\usepackage[flushleft]{threeparttable}

\usepackage{placeins}
\usepackage{lscape}

\newacronym{cnn}{CNN}{Convolutional Neural Network}
\newacronym{nas}{NAS}{Neural Architecture Search}
\newacronym{flops}{FLOPs}{floating point operations}
\newacronym{lra}{LRA}{Low-Rank Approximation}
\newacronym{svd}{SVD}{Singular Value Decomposition}
\newacronym{scef}{DeCEF}{Depthwise Convolutional Eigen-Filter}
\newacronym{conv2d}{Conv2D}{2D convolutional}

  \title{Building Efficient CNNs Using Depthwise Convolutional Eigen-Filters (DeCEF)}

  \author{Yinan Yu \and Samuel Scheidegger \and Tomas McKelvey}
\date{}
\begin{document}

\maketitle
\begin{abstract}
  Deep \glspl{cnn} have been widely used in various domains due to their impressive capabilities. These models are typically composed of a large number of \glspl{conv2d} layers with numerous trainable parameters. 
To reduce the complexity of a network, compression techniques can be applied. These methods typically rely on the analysis of trained deep learning models. However, in some applications, due to reasons such as particular data or system specifications and licensing restrictions, a pre-trained network may not be available. This would require the user to train a \gls{cnn} from scratch. In this paper, we aim to find an alternative parameterization to \gls{conv2d} filters without relying on a pre-trained convolutional network.
During the analysis, we observe that the effective rank of the vectorized \gls{conv2d} filters decreases with respect to the increasing depth in the network, which then leads to the implementation of the \gls{scef} layer.
Essentially, a \gls{scef} layer is a low rank version of the \gls{conv2d} layer with significantly fewer trainable parameters and \gls{flops}. The way we define the effective rank is different from the previous work and it is easy to implement in any deep learning frameworks.
To evaluate the effectiveness of \gls{scef}, experiments are conducted on the benchmark datasets CIFAR-10 and ImageNet using various network architectures. The results have shown a similar or higher accuracy and robustness using about 2/3 of the original parameters and reducing the number of FLOPs to 2/3 of the base network, which is then compared to the state-of-the-art techniques.

\glsresetall


\end{abstract}

\section{Introduction}
Deep \gls{cnn} is one of the most commonly used data-driven techniques. 
Typically, the large number of trainable parameters in deep learning models result in high demands on the computational power and memory capacities, which requires renting or purchasing expensive infrastructure for training. The heat generated by the GPU servers and the high power consumption during training is not environmentally friendly.
Moreover, the size of the network and the number of \gls{flops} play an important role for the inference process, where a small edge device may be used with restrictions on the complexity of the runtime.
Therefore, building an efficient network is beneficial in terms of saving computational resources and reducing the overall cost for deep learning while achieving similar performances.

One topic on constructing an efficient \gls{cnn} is the \gls{nas}, where the focus is to search for an optimal architecture given certain criteria. In this paper, however, we assume that the wiring of the layers is pre-determined. Our focus is on how to improve the efficiency of a \gls{cnn} for a given architecture.

There are mainly two strategies to achieve this.
The first strategy is to take a pre-trained network and reduce the relatively insignificant parameters. This refers to as \emph{compression} or \emph{pruning} in the literature. This is often a desirable approach since many applications are using pre-trained networks as backbones in their networks.

However, in many scenarios, a pre-trained network is not available due to the particular data or system specifications, restrictive licensing, etc, where training a neural network from scratch is inevitable. In this case, after the overall architecture is established, one may re-parameterize the \gls{cnn} to make it more efficient before training. That is, the network is still aimed to accomplish what the original \gls{cnn} is supposed to achieve but with significantly fewer trainable parameters and \gls{flops}. The approximation is on the functional level instead of relying on trained parameters.
This is the main focus of this work. \footnote{Although not being the main focus of this work, the proposed method can also be applied as a compression technique. This aspect is elaborated in the supplementary material.}

The main hypothesis for finding an efficient re-parameterization strategy is that there is significant redundancy in \gls{conv2d} layers, which means that it may be sufficient to express a \gls{conv2d} layer with fewer parameters in order to achieve similar performances.
One of the most commonly used function approximation techniques is the subspace low rank representation (\citet{Belhumeur1997,Jolliffe1986,Golub1996}). It is a family of very well studied and widely used techniques in the area of signal processing and machine learning.
To put it in the context of \gls{cnn}, the main idea is to rearrange the trainable variables into a vector space and find a subspace spanned by the most significant singular vectors of these variables.
This new representation typically results in fewer trainable variables and runtime \gls{flops} with potentially better robustness.

There are two key steps involved to achieve this approximation:
1) find a {\bf representative vector space} for each layer, and
2) estimate the {\bf effective rank} without training.
To find a representative vector space for \gls{conv2d} filters in a \gls{cnn}, we have designed experiments where we observe that 1) vectorized \gls{conv2d} filters exhibit low rank behaviors, and 2) the effective ranks are different for each layer and they have a decreasing tendency with respect to the depth of the network.
Given these observations, we propose a new convolutional filter \gls{scef}. \gls{scef} is parameterized by a new hyperparameter we call rank, where a full rank \gls{scef} is equivalent to a \gls{conv2d} filter, whereas a rank one \gls{scef} is equivalent to a depthwise separable convolutional layer.
To avoid the common problem of over-tuning, we use a rule-based approach for finding the ranks, where the rules are pre-determined by cross-validation on a small dataset trained on a small network. The rules are then applied to larger datasets and networks without further adjustments or tuning.

The paper is organized as follows. First, to motivate our work, we present the experiments and methodologies being used to observe and analyze the low rank behaviors in several trained \glspl{cnn} in Sec.~\ref{sec:motivation}. We then propose the definition of a new type of filter parameterization \gls{scef} in Sec.~\ref{sec:definition}. To further illustrate the advantages of using a \gls{scef} layer, we show two key properties, robustness and complexity, in Sec.~\ref{sec:properties}. In Sec.~\ref{sec:algorithms}, we present the training strategies for \gls{scef}.
In Sec.~\ref{sec:setup}, we show experiments to evaluate the effectiveness of \gls{scef}. First, we run ablation studies on the smaller dataset CIFAR-10 using \gls{scef} to gain empirical insights of its behaviors in Sec.\ref{sec:exp_cifar10}. To further evaluate the two properties of \gls{scef}, we conduct experiments using the benchmark network ResNet-50 1) on ImageNet for comparing complexity versus accuracy; and 2) on the corruption dataset ImageNet-C to validate the robustness of \gls{scef}.
Moreover, in Sec.~\ref{sec:exp_imagenet}, we run further experiments on two additional popular network architectures DenseNet and HRNet. These results are also compared to other state-of-the-art model reduction techniques in Sec.~\ref{sec:compare}.


\section{DeCEF Layers}

\subsection{Motivation}
\label{sec:motivation}
First, let us formally define what a layer is in this context.
\begin{definition}
  \label{def:layer}
  In the scope of this paper, a \gls{conv2d} layer (or a {\bf layer} for short) $$\mathcal{W}=\begin{Bmatrix}\mathbf{w}_j^{(i)}\in\mathbb{R}^{h\times h}: i=1\cdots c_{\text{in}}, j=1\cdots c_{\text{out}}\end{Bmatrix}$$ is a set of trainable units that are characterized by the following attributes: 1) number of input channels $c_{\text{in}}$; 2) number of output channels $c_{\text{out}}$, and 3) parameterization: $\mathbf{w}_{j}^{(i)}\in \mathbb{R}^{ h\times h}$, i.e. the \gls{conv2d} filter.
\end{definition}
Note that there are multiple layers in a network, but we ignore the layer index in this definition for simplicity. When multiple layers appear in the same context, we use $\mathcal{W}_l$ to denote the indexed layer, where the subscript $l\in\{1,\cdots,L\}$ is the layer index and $L$ is the {\bf depth}\footnote{To clarify, this depth refers to the depth of the network. The \emph{depthwise} in \gls{scef} refers to the depth (i.e. input channels) of a layer, which is a different concept.} of the network.
In addition, we denote $K := h^2$. Note that in practice, the filter shape may be rectangular.
Moreover, for the sake of both consistency and convenience, we use $i$ and $j$ to denote the input channel index and the output channel index, respectively.

Our motivation of this work has originated from the low rank behaviors we have observed in the vectorized filter parameters, so let us start with this experimental procedure to illustrate our findings.
\begin{procedure}
  \label{procedure:svd}
  Observing low rank behaviors
  \begin{itemize}
    \item Apply vectorization $\bar{\mathbf{w}}_j^{(i)}:=\operatorname{vec}(\mathbf{w}_j^{(i)}) \in \mathbb{R}^{K}$ and compute the truncated \gls{svd}:
          \begin{equation}
            \label{eqa:trained_sv}
            \bar{\mathbf{U}}^{(i)}\mathbf{S}^{(i)}\mathbf{V}^{(i)\text{T}}=\begin{bmatrix}\bar{\mathbf{w}}^{(i)}_1 &\cdots &\bar{\mathbf{w}}^{(i)}_{c_{\text{out}}}\end{bmatrix},
          \end{equation}
          where matrices $\bar{\mathbf{U}}^{(i)}$ and $\mathbf{V}^{(i)}$ are the left and right singular matrix, respectively; and $\mathbf{S}^{(i)}$ is a diagonal matrix that contains the singular values in a descending order. The implementation of this procedure is well supported by any linear algebra libraries in most programming languages.
    \item Identify the effective rank for each input channel $i$:
          \begin{multline}
            \label{eqa:eff_rank}
            r_i = \mid \{\mathbf{S}^{(i)}[k,k]: ~\mathbf{S}^{(i)}[k,k]\geq \gamma \mathbf{S}^{(i)}[1,1], \\ k=1,\cdots,\min(K, c_{\text{out}}),~\gamma\in[0,1] \}\mid
          \end{multline}
          where $\left|\cdot \right|$ denotes the cardinality of a set and $\mathbf{S}[k,k]$ is the $k^{th}$ diagonal element of matrix $\mathbf{S}$.
    \item The effective rank of one layer $l$:
          \begin{equation}
            \label{eqa:rank_one_layer}
            \hspace{-8mm} \small{r^l = \mid \{s^l: ~s^l\geq \gamma , k=1,\cdots,\min(K, c_{\text{out}}),~\gamma\in[0,1] \}\mid}
          \end{equation}
          where $s^l=\mathbb{E}_i\left(\frac{\mathbf{S}^{(i)}[k,k]}{\mathbf{S}^{(i)}[1,1]}\right)$ and the expected value can be estimated by averaging over all input channels $i$.
  \end{itemize}
\end{procedure}
To illustrate the empirical values, examples can be found in Fig.~[1].
and Fig.~[2] in the supplementary material.
To summarize what we have observed:
\begin{itemize}
  \item[1)] the vectorized \gls{conv2d} filters in a trained \gls{cnn} exhibit low rank properties (cf.~Fig.~[1] in supplementary);
  \item [2)] the effective ranks of vectorized filters show a decreasing tendency when the network goes deeper (cf.~Fig.~[2] in supplementary);
  \item[3)] the effective ranks of vectorized filters converge over training steps (see video in supplementary material).
\end{itemize}
Given these observations, we propose a new layer called \gls{scef} as an alternative parameterization to \gls{conv2d} layers for the purpose of reducing the redundancy.
\subsection{Definition}
\label{sec:definition}
In this section, we introduce the definition of \gls{scef} followed by its two properties.
Generally speaking, subspace techniques bring better robustness to the learning system due to their reduced model complexity.
Motivated by these observations and analyses, we define a \gls{scef} layer as follows.
\begin{definition}[\gls{scef} layer]
  \label{def:scef}
  A \gls{scef} layer is defined by $$\Theta=\begin{Bmatrix}{\bf w}_j^{(i)}, {\bf w}_j^{(i)}\in \mathbb{R}^{h\times h}, i=1\cdots c_{\text{in}}, j=1\cdots c_{\text{out}}\end{Bmatrix}$$ with the following parameterization
  \begin{equation}
    \label{eqa:theta}
    {\bf w}_j^{(i)} = \sum_{k=1}^{r}  a_{k,j}^{(i)}{\bf u}^{(i)}_k, ~r\in [1, h^2]
  \end{equation}
  where $a_{k,j}^{(i)}\in\mathbb{R}$ and ${\bf u}_k^{(i)}\in\mathbb{R}^{h\times h}$, which satisfies $${\bf \bar{u}}_l^{(i)T}{\bf \bar{u}}_m^{(i)} = \begin{cases} 1 & \text{ if } l=m   \\
      0 & \text{ otherwise}
    \end{cases}  $$ for ${\bf \bar{u}}_k^{(i)}=\operatorname{vec}({\bf u}_k^{(i)})\in\mathbb{R}^{h^2}$. The parameters ${\bf u}_k^{(i)}\in\mathbb{R}^{h\times h}$ are called the {\bf eigen-filters}.

\end{definition}
Note that for the sake of clarity, we use $\Theta$ to denote the \gls{scef} layer, instead of the generic notation $\mathcal{W}$ in Def.\ref{def:layer}.

\subsection{Properties}
\label{sec:properties}
In this section, we present two key properties of the \gls{scef} layer. These properties are then empirically evaluated in the experiment section.
\begin{properties} Complexity (one layer)
  \label{properties:complexity}
  \begin{itemize}
    \item Number of trainable parameters ($N$)
          \begin{itemize}
            \item $N$(\gls{conv2d}): $c_{\text{in}}c_{\text{out}}h^2$
            \item $N$(\gls{scef}): : $N_u+N_a$, where
                  \begin{itemize}
                    \item[-] Eigen-filters:  $N_u=c_{\text{in}}h^2r$
                    \item[-] Coefficients: $N_a=c_{\text{in}}c_{\text{out}}r$
                  \end{itemize}
          \end{itemize}
          For $r=h^2$, it is trivial to randomly initialize eigen-filters that span the whole $h^2$ dimensional vector space and hence the eigen-filters do not need to be trainable, i.e. $N_u=0$ and $N_a=c_{\text{in}}c_{\text{out}}h^2$. Therefore, \gls{conv2d} and \gls{scef} are equivalent for $r=h^2$.

          For $r<h^2$, $N$(\gls{scef})$<$$N$(\gls{conv2d}) if \small{$r\leq \left \lfloor{\frac{c_{\text{out}}h^2}{c_{\text{out}}+h^2}}\right \rfloor$}.

            Example. Given $c_{\text{in}}=c_{\text{out}}=128$ and $h=3$, we have $N$(\gls{conv2d})$=147456$. If \small{$r\leq 8<h^2=9$}, then $N(\text{\gls{scef}}) < N(\text{Conv2D})$. For $r= 8$, $N$(\gls{scef})$=140288$ and for $r= 4$, $N$(\gls{scef})$=70144$.
    \item FLOPs ($F$)

          We count the multiply-accumulate operations (macc) and we do not include bias in our calculations. Given the dimension of the input layer $ H \times W\times c_{\text{in}}$, let $t=\left\lfloor{\frac{H}{stride}}\right\rfloor\times\left\lfloor{\frac{W}{stride}}\right\rfloor$,
          \begin{itemize}
            \item $F$(\gls{conv2d}): $th^2c_{\text{in}}c_{\text{out}}$

            \item $F$(\gls{scef}): $tc_{\text{in}}r\left(h^2+c_{\text{out}}\right)$
          \end{itemize}
          Example. Given $H=W=100$, $c_{\text{in}}=128$, $c_{\text{out}}=128$ and $h=3$ with $stride=1$, we have $F$(\gls{conv2d})$=1.47~\text{GFLOPs}$. For $r=8$, $F$(\gls{scef})$=1.40~\text{GFLOPs}$. For $r=4$, $F$(\gls{scef})$=0.70~\text{GFLOPs}$.
  \end{itemize}

\end{properties}

\begin{properties} Robustness
  \begin{lemma}
    \label{lemma:reg}
    Let $\Delta\mathbf{I}_i$ be an additive perturbation matrix and $\mathbf{w}_{j}^{(i)}\in \mathbb{R}^{h\times h}$ be a filter parameterized by Eq.~\eqref{eqa:theta}, which is learned from some training process.
    Let
    \begin{equation}
      \label{eqa:Ubar}
      \bar{\mathbf{U}}^{(i)}=\left[{\bf \bar{u}}_0^{(i)},\cdots, {\bf \bar{u}}_r^{(i)}\right].
    \end{equation}
    If $\bar{\mathbf{U}}^{(i)\text{T}}\bar{\mathbf{U}}^{(i)}=\mathbf{I}$ and $\left|\left|\mathbf{a}^{(i)}_{j}\right|\right|_2\leq \epsilon$, $\forall i, j$,
    \begin{equation}
      \label{eqa:reg}
      \left|\left|\sum_{i}\Delta\mathbf{I}_i*\mathbf{w}_{j}^{(i)}\right|\right|_{\infty}\leq \epsilon hr\sum_{i}\left|\left|\Delta\mathbf{I}_i\right|\right|_2.
    \end{equation}
  \end{lemma}
\end{properties}

\begin{proof}
  See supplementary material.
\end{proof}
Robustness in this context is indicated by the propagation of the additive perturbation between input and output feature maps.
Lemma \ref{lemma:reg} shows that when (1) $\bar{\mathbf{U}}^{(i)\text{T}}\bar{\mathbf{U}}^{(i)}=\mathbf{I}$, i.e. the vectorized filters are orthonormal, and (2) $\left|\left|\mathbf{a}^{(i)}_{j}\right|\right|_2\leq \epsilon$, i.e. the coefficients are bounded by $\epsilon$,
the effect of the perturbation on the output is bounded by Eq.~\eqref{eqa:reg}.

The rank $r$ of the eigen-filters is a hyperparameter that yields a trade-off between the robustness and the representational power of a \gls{scef} layer. In this work, we use a rule based approach for choosing this hyperparameter.

\subsection{Training algorithms}
\label{sec:algorithms}
In this section, we show how to construct and train a network composed of \gls{scef} layers.

\subsubsection{The optimization problem}
Given a network architecture with a set of layers $\mathcal{N}=\begin{Bmatrix}\mathcal{W}_1&\cdots&\mathcal{W}_L\end{Bmatrix}$.
Denote the index set of the network $\mathcal{N}$ using $\mathcal{I}_{\mathcal{N}}=\{1,\cdots, L\}$.
Let $\mathcal{G}=\begin{Bmatrix}\Theta_{m_1}&\cdots&\Theta_{m_S}\end{Bmatrix}\subseteq \mathcal{N}$ be a set of \gls{scef} layers with index set $\mathcal{I}_{\mathcal{G}}=\{m_1,\cdots, m_S\}$. Let $\tilde{\mathcal{G}}=\mathcal{N}\setminus\mathcal{G}$ be the rest of the layers in the network. Let $f(\mathcal{N})$ be an objective function and $\lambda\Phi(\tilde{\mathcal{G}})$ be a regularization term applied to the set $\tilde{\mathcal{G}}$, where $\lambda>0$ is the multiplier. The optimization problem is formulated as:
\begin{eqnarray}
  \nonumber
  \min &~&f(\mathcal{N} )+\lambda\Phi(\tilde{\mathcal{G}})\\
  \mathrm {subject~to} &~& \bar{\mathbf{U}_l}^{(i)\text{T}}\bar{\mathbf{U}_l}^{(i)}=\mathbf{I}\\
  \nonumber
  &~&\left|\left|\mathbf{a}^{(i)}_{l, j}\right|\right|_2\leq \epsilon,~
  \forall l\in\mathcal{I}_{\mathcal{G}}
\end{eqnarray}

\subsubsection{Relaxed regularization}
Finding an exact optimal \gls{scef} layer is an NP-hard problem due to the orthonormality constraint. Therefore, we approximate the constraint by the following regularizations.
For a given \gls{scef} layer $l$, we have:
\begin{eqnarray}
  \label{eqa:reg1}
  \Phi_1: &&
  \lambda_1\left|\left|\bar{\mathbf{U}}^{(i)\text{T}}\bar{\mathbf{U}}^{(i)}-\mathbf{I}\right|\right|_2\\
  \label{eqa:reg2}
  \Phi_2: &&
  \lambda_2\left|\left|\mathbf{a}^{(i)}_{j}\right|\right|_2,~~ \mathbf{a}^{(i)}_{j}=\begin{bmatrix}a_{1,j}^{(i)},\cdots,a_{r,j}^{(i)}\end{bmatrix}
\end{eqnarray}
Note that the layer index $l$ is neglected.

The loss function of the whole network is then written as:
\begin{equation}
  \label{eqa:optimization2}
  f(\mathcal{N} )+\lambda\Phi(\tilde{\mathcal{G}})+\lambda_1\Phi_1(\mathcal{G})+\lambda_2\Phi_2(\mathcal{G})
\end{equation}
where $\Phi_i(\mathcal{G})=\sum_{l\in\mathcal{I}_{\mathcal{G}}}\Phi_i^l$, $i=1,2$.
\subsubsection{Deterministic rule-based hyperparameters}
Hyperparameters are chosen based on deterministic rules to avoid complex hyperparameter tuning and to increase reproducibility. These rules are determined using a transfer learning approach. First, we find the hyperparameters in \gls{scef} using cross-validation on a small dataset CIFAR-10, where cross-validation is affordable. Then we establish a deterministic rule for each hyperparameter. These rules are then directly applied to the larger dataset ImageNet without tuning.
There are three sets of hyperparameters h1 $\sim$ h3:
\begin{itemize}
  \item[{\bf h1:}] Ranks $r$ (Algo.~\ref{alg:usecase1}): The observation of singular values from several networks shows that the effective ranks typically have a decreasing trend with respect to the depth, i.e., layers at the beginning of the network often have higher rank, and vice versa. The idea of choosing the rank before training a network is to find a monotonically decreasing function given the increasing depth.
        In this paper, we adopt two alternative routines for choosing the rank in each layer: linear decay (simple) and logarithmic decay (aggressive). 
        Let $l$ be the depth index of a layer and $K=h^2$. Denote $l_{\text{max}}=\max(l)$ and $l_{\text{min}}=\min(l)$.
        \begin{itemize}
          \item Linear decay: $\hat{r}_i =\left \lfloor{K-\frac{l(K-1)}{l_{\text{max}}-l_{\text{min}}}}\right \rfloor $.
          \item Logarithmic decay: $\hat{r}_i = \left \lfloor{\frac{K-1}{log2(l+1)}}\right \rfloor $.
        \end{itemize}
  \item[{\bf h2:}] Regularization coefficients (Algo.\ref{alg:usecase1}, cf.~Eq.~\eqref{eqa:reg1}, \eqref{eqa:reg2}): $\lambda_1=0.0001r$ and $\lambda_2=0.0001$.
  \item[{\bf h3:}] Singular value threshold to determine the effective rank (cf.~Eq.~\eqref{eqa:eff_rank}, \eqref{eqa:rank_one_layer}): $\gamma=0.3$.
\end{itemize}
The training algorithm is summarized in Alg.~\ref{alg:usecase1}.
\begin{algorithm}(\gls{scef} training strategy)
  \label{alg:usecase1}
  \begin{itemize}
    \item {\bf Step 1:} Choose a network topology fully or partially composed of \gls{scef} layers. For example, one can replace all \gls{conv2d} layers with \gls{scef} layers.
    \item {\bf Step 2:} Choose hyperparameters $r, \lambda_1, \lambda_2$.
    \item {\bf Step 3:} Initialization for each \gls{scef} layer ($k=1, \cdots, r$):
          \begin{itemize}
            \item Eigen-filters $\mathbf{u}_{k}^{(i)}$:
                  \begin{itemize}
                    \item[-] Generate random matrices: $\mathbf{A}^{(i)}\in\mathbb{R}^{K\times r}$.
                    \item[-] Compute the truncated SVD: $\mathbf{A}^{(i)} = \bar{\mathbf{U}}^{(i)}\bar{\mathbf{S}}^{(i)}\bar{\mathbf{V}}^{(i)\text{T}}$.
                    \item[-] Reshape each column in $\bar{\mathbf{U}}^{(i)}$ into matrix $\mathbf{u}_{k}^{(i)}\in\mathbb{R}^{K}$.
                  \end{itemize}
            \item Coefficients $a_{k,j}^{(i)}$: randomly initialized from a normal distribution.
          \end{itemize}
    \item[] {\bf Step 4:} Forward and backward paths:
          \begin{itemize}
            \item Forward ${\bf I}_l\rightarrow {\bf I}_{l+1}$: for each out channel $j$,
                  $${\bf I}_{l+1}^j= \sum_{i=1}^{c_{\text{in}}}\sum_{k=1}^{r}  a_{k,j}^{(i),l}{\bf u}^{(i),l}_k*{\bf I}_{l}^i$$
            \item Backward: backpropagation with the loss function described in Eq.~\eqref{eqa:optimization2}.
          \end{itemize}
  \end{itemize}
\end{algorithm}


\subsection{Refactor a \gls{conv2d} network into \gls{scef}}
There are such use cases where a pre-trained \gls{cnn} is available and one needs to reduce the runtime complexity of the network. This is not the focus of this work but we also propose a compression algorithm presented in the supplementary material.

\section{Related Work}
To compare to the state-of-the-art techniques, in this section, we list the following existing approaches. More related work can be found in the supplementary material.

    {\bf Subspace techniques:}
The first category is the \gls{lra} technique.
There are mainly two different approaches in the existing literature:
1) \emph{Separable bases:}
\citet{Jaderberg2014} decomposes the $d\times d$ filters into $1\times d$ and $d\times 1$ filters to construct rank-1 bases in the spatial domain.
In later work \cite{Tai2015}\cite{lin2018holistic}, closed form solutions that significantly improves the efficiency over previous iterative optimization solvers are proposed.
\citet{Ioannou2015} introduces a novel weight initialization that allows small basis filters to be trained from scratch, which has achieved similar or higher accuracy than the conventional \glspl{cnn}.
\citet{Yu2017} proposes a SVD-free algorithm that uses the idea that filters usually share smooth components in a low-rank subspace.
\citet{alvarez2017compression} introduces a regularizer that encourages the weights of the layers to have low rank during the training.
2) \emph{Filter vectorization:}
Some existing work implements the low rank approximation by vectorizing the filters. For instance, \citet{Denton2014} stacks all filters for each output channel into a high dimensional vector space and approximates the trained filters using \gls{svd}.
\citet{Wen2017} presents a regularization to enforce filters to coordinate into lower-rank space, where the subspaces are constructed from all the input channels for each given output channel.
Recently, \citet{Peng2018} proposed a decomposition focusing on exploiting the filter group structure for each layer.

    {\bf Pruning:}
Pruning refers to techniques that aim at reducing the number of parameters in a pre-trained network by identifying and removing redundant weights. This is a very invested topic in the attempt to reduce the model complexity. Although being different from our use case, we list the state-of-the-art pruning techniques in this section to have a more complete view on model reduction techniques.
In Optimal Brain Damage by \citet{LeCun1990}, and later in Optimal Brain Surgeon by \citet{Hassibi1993}, redundant weights are defined by their impact on the objective function, which are identified using the Hessian of the loss function.
Other definitions of redundancy have been proposed in subsequent work.
For instance, \citet{Anwar2017} applies pruning on the filter-level of \glspl{cnn} by using particle filters to propose pruning candidates.
\citet{Han2015b} introduces a simpler pruning method using a strong L2 regularization term, where weighs under a certain threshold are removed.
\citet{Molchanov2016} uses Taylor expansion to approximate the influence in the loss function by removing each filter.
\citet{hu2016network} iteratively optimizes the network by pruning unimportant neurons based on analysis of their outputs on a large dataset.
\citet{Li2016} identifies and removes filters having a small effect on the accuracy.
\citet{aghasi2017net} prunes a trained network layer-wise by solving a convex optimization program.
\citet{liu2017learning} takes wide and large networks as input models, but during training insignificant channels are automatically identified and pruned afterwards.
More recently, in~\cite{Luo2017b} and \cite{Luo2017}, \citeauthor{Luo2017} analyzes the redundancy of filters in a trained network by looking at statistics computed from its next layer.
\citet{He2017} proposes an iterative LASSO regression based channel selection algorithm.
\citet{Huang2018} removes filters by training a pruning agent to make decisions for a given reward function.
In~\cite{Yu2018}, \citeauthor{Yu2018} poses the pruning problem as a binary integer optimization and derives a closed-form solution based on final response importance.
\citet{lin2018accelerating} prunes filters across all layers by proposing a global discriminative function based on prior knowledge of each filter.
\citet{tung2018clip} combines network pruning and weight quantization in a single learning framework that performs pruning and quantization jointly.
\citet{zhang2018systematic} first formulate the weight pruning problem a a nonconvex optimization problem constraints specifying the sparsity requirements and optimize using the using the alternating direction method of multipliers.
Other work \cite{zhuang2018discrimination} uses discrimination-aware losses into the network to increase the discriminative power of intermediate layers.
\cite{huang2018data} adds a scaling factor to the outputs and then add sparsity regularizations on these factors.
\citet{he2019filter} compresses CNN models by pruning filters with redundancy, rather than those with "relatively less" importance.
\citet{lin2019toward} propos a scheme that incorporates two different regularizers which fully coordinates the global output and local pruning operations to adaptively prune filters.
Later, \citet{lin2019towards}, proposed an effective structured pruning approach that jointly prunes filters as well as other structures in an end-to-end manner by defining a new objective function with sparsity regularization which is solved by generative adversarial learning.
\citet{ding2019centripetal} proposes a novel optimization method, which can train several filters to collapse into a single point in the parameter hyperspace which can be trimmed with no performance loss.
\citet{liu2019metapruning} proposes a meta network, which is able to generate weight parameters for any pruned structure given the target network, which can be used to search for good-performing pruned networks.
\cite{you2019gate} introduce gate decorators to identify unimportant filters to prune.
\cite{Molchanov_2019_CVPR} prunes filters by using Taylor expansions to approximate a filter’s contribution.
\cite{ding19a} finds the least important filters to prune by a binary search.
which keeps searching for the least important filters in a binary search manner
\citet{luo2020autopruner} proposes an efficient channel selection layer to find less important filters automatically in a joint training manner.
\citet{lin2020hrank} proposes a method that is mathematically formulated to prune filters with low-rank feature maps.
\cite{He_2020_CVPR} introduces a differentiable pruning criteria sampler.
\cite{ding2020lossless} proposes a re-parameterization of CNNs to a remembering part and a forgetting part. The former learns to maintain the performance and the latter learns for efficiency.
\cite{liu21ab} proposes a layer grouping algorithm to find coupled channels automatically.
\cite{shi2021} uses an effective estimation of each filter, i.e., saliency, to measured filters from two aspects: the importance for prediction performance and the consumed computational resources. This can be used to preserve the prediction performance while zeroing out more computation-heavy filters.

    {\bf Architectural design:}
Effort has been put into designing a smaller network architecture without loss of the generalization ability. For instance, \citet{He2016} achieves a higher accuracy in \cite{He2016,He2016b}  compared to other more complex networks by introducing the residual building block. The residual building blocks adds an identity mapping that allows the signals to be directly propagated between the layers.
\citet{Iandola2016} introduces SqueezeNet and the Fire module, which is designed to reduce the number of parameters in a network by introducing $1\times1$ filters.
By utilizing dense connections pattern between blocks, \citet{huang2016densely} manages to reduce the number of required parameters.
\citet{Xie2017} proposed a multi-branch architecture which exposes a new hyperparameter for each block to control the capacity of the network.
Other work, like MobileNet\cite{howard2017mobilenets,Sandler2018} and EfficientNet\cite{tan2019efficientnet} specifically focus on builing architectures suitable for devices with low compute capacity, such as mobile phones.
By a design that maintains a high-resolution representation throughout the whole network, \citet{wang2020deep} achieves good accuracy and performance in HRNet.

    {\bf Compression:}
Deep Compression, by \citet{Han2015a}, reduces the storage size of the model using quantization and Huffman encoding to compress the weights in the network. Other work on reducing the memory size of models is done by binarization. In XNOR-Net by \citet{Rastegari2016}, the weights are reduced to a binary representation and convolutions are replaced by XNOR operations.
More recently, \citet{Suau2018} proposed to analyze filter responses to automatically select compression methods for each layer.

    {\bf Weight sharing:}
Another approach to reduce the number of parameters in a network is to share weights between the filters and layers. \citet{Boulch2018} share weights between the layers in a residual network operating on the same scale.

    {\bf Depthwise separable convolutions:}
introduced by \citet{Chollet2017}, have shown to be a more efficient use of parameters compared to regular \glspl{conv2d} Inception like architectures. Depthwise separable convolutions have also been used in other work, e.g., in~\cite{He2017} by \citeauthor{He2017}, where it was used to gain a computational speed-up of ResNet networks.

    {\bf Our focus:} We observe and analyze the \gls{conv2d} from a different perspective compared to the previous subspace techniques. More specifically, i) we vectorize the filters instead of using separable basis in the original vector space (\citet{Jaderberg2014}, \citet{Tai2015}, \citet{Ioannou2015}, \citet{Yu2017}); ii) we do not concatenate these vectorized filters into a large vector space (\citet{Denton2014}, \citet{Wen2017}, \citet{Peng2018}), which achieves a better modularity compared to the concatenated vectors. Our perspective is motivated by the empirical evidence from our experiments. This opens up new opportunities and provides new analytical tools for understanding the design of convolutional networks with respect to their subspace redundancies.
In our experiments, we choose a popular base network (ResNet) and compare our experimental results to various modifications of the same base network.
We also conduct tests on other more recent network architectures such as HRNet-W18-C and DenseNet-121 for further comparison and validation.


\section{Experiments and Results}
\label{sec:setup}
\subsection{Hardware}
For training and experiments, Nvidia Tesla V100 SXM2 with \SI{32}{\giga\byte} of GPU memory are used.
\subsection{Dataset CIFAR-10: Ablation Study}
\label{sec:exp_cifar10}
 {\bf Dataset:} To empirically study the behavior of \gls{scef}, we conduct various experiments on the standard image recognition dataset CIFAR-10 by~\citet{Krizhevsky2009}.
 {\bf Benchmark:} We use ResNet-32 as the base net for comparison.
        ResNet-32 has three {\bf blocks}, where the last block (block-3) in ResNet-32 has the most filters.
        Since our goal is to reduce the amount of trainable parameters and FLOPs, we mainly vary the structure in block-3 in our experiments.\\
  {\bf Experiments:}
  We design four experiments as follows.
  \paragraph{Experiment 1.~Varying rank $r$ and $c_{\text{out}}$}
                For a layer with input channels $i=1,\cdots,c_{\text{in}}$ and output channels $j=1,\cdots,c_{\text{out}}$, the filters in the \gls{scef} layer is expressed as $\mathbf{w}_j^{(i)}=\sum_{k=1}^{r}a_{k,j}^{(i)}\mathbf{u}_{k}^{(i)}$. We empirically show that \gls{scef} layers achieve higher accuracy with significantly lower number of parameters.
                In this experiment, we vary two hyperparameters: 1) the rank $r$ of each filter in the \gls{scef} layer, and 2) the number of output channels $c_{\text{out}}$.
                We compare the accuracy versus the number of parameters in different types of layers (\gls{conv2d} and \gls{scef} with different hyperparameters). As shown in Fig.~\ref{fig:param_cmp_plot}, with a lower number of parameters, \gls{scef} achieves a better accuracy with low rank techniques.
                Moreover, when we increase the number of output channels, \gls{scef} shows a even more promising result with fewer parameters in total.
\begin{figure}[ht!]
    \centering
    \includegraphics[width=0.9\linewidth]{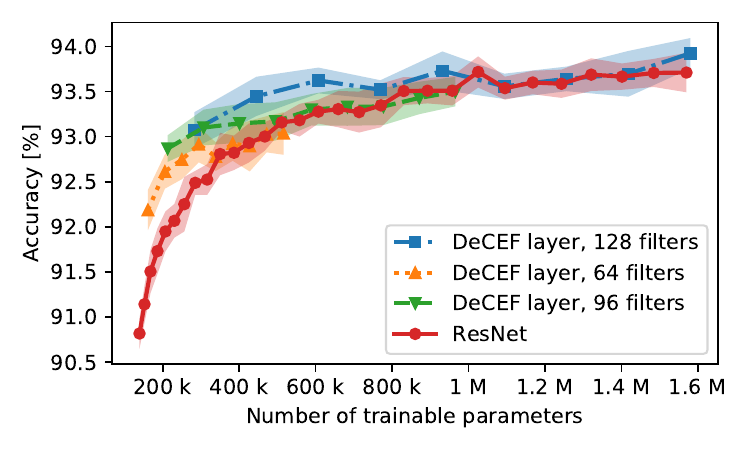}
    \caption{Accuracy versus number of parameters on CIFAR-10.}
    \label{fig:param_cmp_plot}
  \end{figure}
          \paragraph{Experiment 2.~Trainable vs frozen eigen-filters}
                In Algo.~\ref{alg:usecase1}, the eigen-filters in \gls{scef} layers are trained simultaneously using backpropagation. In this experiment, we investigate the impact of this training process and try to understand if it is sufficient to use \emph{random basis vectors} as eigen-filters.
                We initialize the eigen-filters according to Algo.~\ref{alg:usecase1} and freeze them during training.
                The comparison between the accuracies achieved by frozen and trainable eigen-filters can be found in Fig.~\ref{fig:reg_frozen_plot}.
                By using frozen eigen-filters, the network has a fewer number of trainable parameters for the same rank. With a low rank ($r<5$), the accuracy is degraded without training.
  \begin{figure}[ht!]
    \centering
    \includegraphics[width=0.9\linewidth]{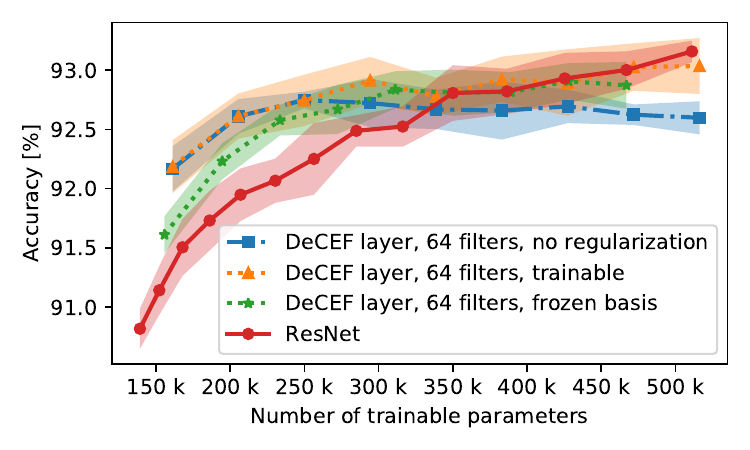}
    \caption{\gls{scef} layer with trainable vs frozen bases on CIFAR-10.}
    \label{fig:reg_frozen_plot}
  \end{figure}
          \paragraph{Experiment 3.~With or without $\Phi_1$ regularization}
                To study the effect of $\Phi_1$ introduced in Eq.~\eqref{eqa:reg1}, some experiments can be found in Fig.~\ref{fig:reg_frozen_plot}.
                We can see that with a high rank, the regularization needs to be applied.
                In our experiment, we use $\lambda_1=0.0001r$ and $\lambda_2=0.0001$, where $\lambda_1$ is the multiplier of the constraint on the eigen-filters and $\lambda_2$ is on the subspace coefficients. The reason for having the multiplier $r$ in  $\lambda_1$ is to suppress the growth of the cost when $r$ becomes large.
          \paragraph{Experiment 4.~Comparison to related work}
          In this experiment, we implement Algo.~\ref{alg:usecase1} (\gls{scef}-ResNet-32) to compared to the state-of-the-art techniques. We vary the number of output channels $c_{\text{out}}$ in the last ResNet block for comparison, where we see that having fewer eigen-filters with more output channels yields a better result.

        {\bf Results:} The results are presented in terms of the estimated mean and the standard deviation of the classification accuracy on the testing set with 10 runs for each experimental setup, which are shown in Fig.~\ref{fig:param_cmp_plot} and Fig.~\ref{fig:reg_frozen_plot}.
        The accuracy is then presented with respect to the number of trainable parameters for each network structure.
        For \gls{scef} layers, there are nine data points in each presented result, which correspond to different layer ranks in block-3 $r_3\in \{1,\cdots, 9\}$. In addition, the number of trainable parameters in \gls{scef} layers is also varied by using different numbers of output channels in block-3, i.e., $c_{\text{out}}\in\{64,96,128\}$. We then vary $c_{\text{out}}$ in ResNet-32 block-3 ($c_{\text{out}}\in\{16,20,24,...,128\}$)  to have a comparable result.
        We compare the accuracy achieved by \gls{scef}-ResNet-32 Fig.~\ref{fig:cifar_comparison}. Additional results can be found in the supplementary material.
\begin{table}[h!]
  \begin{center}
    \resizebox{\columnwidth}{!}{\begin{tabular}{lrrr}\toprule
Network & Acc. & No. param. & MFLOPs\\
\midrule\multicolumn{4}{l}{{\bf (a) \Glsentryshort{scef} vs baseline network}} \\\hline
DeCEF-ResNet-32 (32, 64, 128)\textsuperscript{0}  & {\SI{94.19}{\percent}} & {\SI{533.00}{\kilo\nothing}} & \SI{108.00}{\nothing} \\
DeCEF-ResNet-32 (24, 48, 96)\textsuperscript{1}  & {\SI{93.64}{\percent}} & {\SI{311.00}{\kilo\nothing}} & \SI{64.72}{\nothing} \\
ResNet-56\textsuperscript{2} \cite{He2016} & {\SI{93.03}{\percent}} & {\SI{850.00}{\kilo\nothing}} & \SI{125.49}{\nothing} \\
ResNet-32\textsuperscript{3} \cite{He2016} & {\SI{92.49}{\percent}} & {\SI{467.00}{\kilo\nothing}} & \SI{69.00}{\nothing} \\
DeCEF-ResNet-32 (16, 32, 64)\textsuperscript{4}  & {\SI{92.45}{\percent}} & {\SI{148.00}{\kilo\nothing}} & \SI{32.42}{\nothing} \\
\midrule\multicolumn{4}{l}{{\bf (b) Related work}} \\\hline
ResRep ResNet-110\textsuperscript{5} \cite{ding2020lossless} & {\SI{94.62}{\percent}} & {} & \SI{105.68}{\nothing} \\
C-SGD-5/8 ResNet-110\textsuperscript{6} \cite{ding2019centripetal} & {\SI{94.44}{\percent}} & {} & \SI{98.91}{\nothing} \\
HRank ResNet-110 1\textsuperscript{7} \cite{lin2020hrank} & {\SI{94.23}{\percent}} & {\SI{1.04}{\mega\nothing}} & \SI{148.70}{\nothing} \\
SASL ResNet-110\textsuperscript{8} \cite{shi2021} & {\SI{93.99}{\percent}} & {\SI{1.17}{\mega\nothing}} & \SI{122.15}{\nothing} \\
SFP ResNet-56 10\%\textsuperscript{9} \cite{he2018soft} & {\SI{93.89}{\percent}} & {} & \SI{107.00}{\nothing} \\
SASL ResNet-56\textsuperscript{10} \cite{shi2021} & {\SI{93.88}{\percent}} & {\SI{689.35}{\kilo\nothing}} & \SI{80.44}{\nothing} \\
SFP ResNet-110 30\%\textsuperscript{11} \cite{he2018soft} & {\SI{93.86}{\percent}} & {} & \SI{150.00}{\nothing} \\
SASL* ResNet-110\textsuperscript{12} \cite{shi2021} & {\SI{93.80}{\percent}} & {\SI{786.04}{\kilo\nothing}} & \SI{75.36}{\nothing} \\
LFPC ResNet-110\textsuperscript{13} \cite{He_2020_CVPR} & {\SI{93.79}{\percent}} & {} & \SI{101.00}{\nothing} \\
SFP ResNet-56 30\%\textsuperscript{14} \cite{he2018soft} & {\SI{93.78}{\percent}} & {} & \SI{74.00}{\nothing} \\
FPGM-only 40\% ResNet-110\textsuperscript{15} \cite{he2019filter} & {\SI{93.74}{\percent}} & {} & \SI{121.00}{\nothing} \\
ResRep ResNet-56 1\textsuperscript{16} \cite{ding2020lossless} & {\SI{93.73}{\percent}} & {} & \SI{59.09}{\nothing} \\
LFPC ResNet-56 1\textsuperscript{17} \cite{He_2020_CVPR} & {\SI{93.72}{\percent}} & {} & \SI{66.40}{\nothing} \\
SASL* ResNet-56\textsuperscript{18} \cite{shi2021} & {\SI{93.58}{\percent}} & {\SI{538.90}{\kilo\nothing}} & \SI{53.84}{\nothing} \\
HRank ResNet-56 1\textsuperscript{19} \cite{lin2020hrank} & {\SI{93.52}{\percent}} & {\SI{710.00}{\kilo\nothing}} & \SI{88.72}{\nothing} \\
FPGM-only 40\% ResNet-56\textsuperscript{20} \cite{he2019filter} & {\SI{93.49}{\percent}} & {} & \SI{59.40}{\nothing} \\
SFP ResNet-56 20\%\textsuperscript{21} \cite{he2018soft} & {\SI{93.47}{\percent}} & {} & \SI{89.80}{\nothing} \\
C-SGD-5/8 ResNet-56\textsuperscript{22} \cite{ding2019centripetal} & {\SI{93.44}{\percent}} & {} & \SI{49.13}{\nothing} \\
GBN-40\textsuperscript{23} \cite{you2019gate} & {\SI{93.43}{\percent}} & {\SI{395.25}{\kilo\nothing}} & \SI{50.07}{\nothing} \\
GAL-0.6 ResNet-56\textsuperscript{24} \cite{lin2019towards} & {\SI{93.38}{\percent}} & {\SI{750.00}{\kilo\nothing}} & \SI{78.30}{\nothing} \\
HRank ResNet-110 2\textsuperscript{25} \cite{lin2020hrank} & {\SI{93.36}{\percent}} & {\SI{700.00}{\kilo\nothing}} & \SI{105.70}{\nothing} \\
SFP ResNet-56 40\%\textsuperscript{26} \cite{he2018soft} & {\SI{93.35}{\percent}} & {} & \SI{59.40}{\nothing} \\
LFPC ResNet-56 2\textsuperscript{27} \cite{He_2020_CVPR} & {\SI{93.34}{\percent}} & {} & \SI{59.10}{\nothing} \\
SFP ResNet-32 10\%\textsuperscript{28} \cite{he2018soft} & {\SI{93.22}{\percent}} & {} & \SI{58.60}{\nothing} \\
HRank ResNet-56 2\textsuperscript{29} \cite{lin2020hrank} & {\SI{93.17}{\percent}} & {\SI{490.00}{\kilo\nothing}} & \SI{62.72}{\nothing} \\
ResNet-56-pruned-A\textsuperscript{30} \cite{Li2016} & {\SI{93.10}{\percent}} & {\SI{770.10}{\kilo\nothing}} & \SI{112.00}{\nothing} \\
GBN-30\textsuperscript{31} \cite{you2019gate} & {\SI{93.07}{\percent}} & {\SI{283.05}{\kilo\nothing}} & \SI{37.27}{\nothing} \\
ResNet-56-pruned-B\textsuperscript{32} \cite{Li2016} & {\SI{93.06}{\percent}} & {\SI{733.55}{\kilo\nothing}} & \SI{90.90}{\nothing} \\
ResNet-110-pruned-B\textsuperscript{33} \cite{Li2016} & {\SI{93.00}{\percent}} & {\SI{1.16}{\mega\nothing}} & \SI{115.00}{\nothing} \\
NISP-56\textsuperscript{34} \cite{Yu2018} & {\SI{92.99}{\percent}} & {\SI{487.90}{\kilo\nothing}} & \SI{81.00}{\nothing} \\
FPGM-mix 40\% ResNet-32\textsuperscript{35} \cite{he2019filter} & {\SI{92.82}{\percent}} & {} & \SI{32.30}{\nothing} \\
GAL-0.5 ResNet-110\textsuperscript{36} \cite{lin2019towards} & {\SI{92.74}{\percent}} & {\SI{950.00}{\kilo\nothing}} & \SI{130.20}{\nothing} \\
ResRep ResNet-56 2\textsuperscript{37} \cite{ding2020lossless} & {\SI{92.67}{\percent}} & {} & \SI{27.82}{\nothing} \\
HRank ResNet-110 3\textsuperscript{38} \cite{lin2020hrank} & {\SI{92.65}{\percent}} & {\SI{530.00}{\kilo\nothing}} & \SI{79.30}{\nothing} \\
LFPC ResNet-32\textsuperscript{39} \cite{He_2020_CVPR} & {\SI{92.12}{\percent}} & {} & \SI{32.70}{\nothing} \\
\end{tabular}
}
  \end{center}
  \caption{Comparison to state-of-the-art model reduction techniques on CIFAR-10.}
  \label{tab:cifar10_results}
\end{table}

\subsection{Dataset ImageNet (ILSVRC-2012)}
\label{sec:exp_imagenet}
\subsubsection{Accuracy versus complexity}
\label{sec:compare}
To further compare our algorithms to the state-of-the-art, we use the standard dataset ImageNet (ILSVRC-2012) by~\citet{Deng2009}.
ImageNet has \SI{1.2}{\mega\nothing} training images and \SI{50}{\kilo\nothing} validation images of 1000 object classes, commonly evaluated by Top-1 and Top-5 accuracy.
We use the networks ResNet-50 v2~\citet{He2016b}, DenseNet-121~\citet{huang2016densely} and HRNet-W18-C~\citet{wang2020deep} as the base networks. 
The results are visualized in Fig.~\ref{fig:imagenet_comparison} for Top-1 accuracy (Top-5 accuracy can be found in the supplementary material).
   \begin{figure}[h!]
    \centering
    \includegraphics[width=1.0\linewidth]{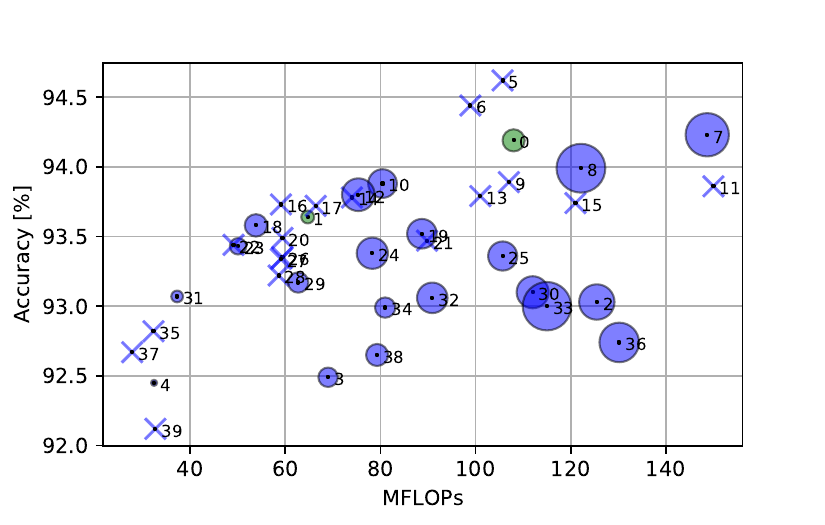}
    \caption{Ball chart for CIFAR-10, where the size of the ball indicates the number of trainable parameters. For papers that have not reported the \gls{flops}, we use a cross instead of a ball to represent them. The exact values are reported in Table~\ref{tab:cifar10_results}. The number in each ball is the network ID, which is indicated as the superscript of each entry in Table~\ref{tab:cifar10_results}.}
    \label{fig:cifar_comparison}
  \end{figure}

The hyperparameters used in \gls{scef}-ResNet-50 are determined by the deterministic rules presented in {\bf h1}, {\bf h2} and {\bf h3}. 
For each setup, we have five runs and report the average accuracy and its standard deviation in the supplementary material.
  \begin{figure}[ht!]
    \centering
    \includegraphics[width=1.0\linewidth]{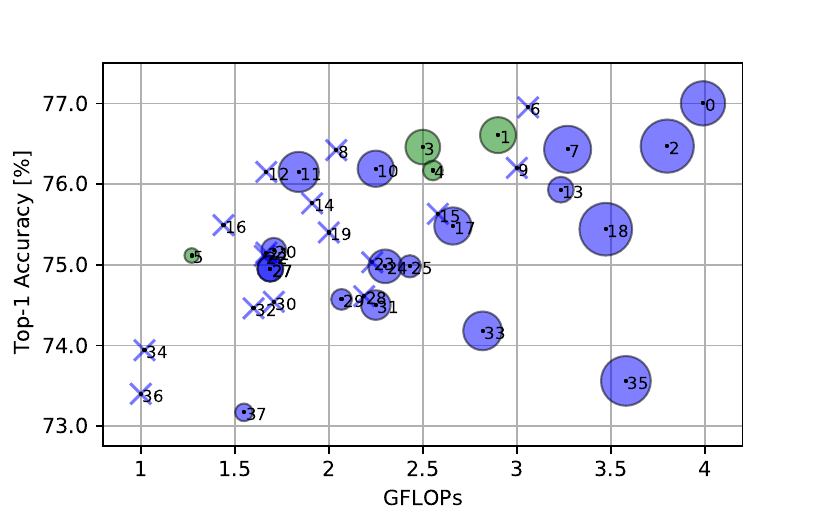}
    \caption{Ball chart for ImageNet Top-1 accuracy with the same set up as Fig.~\ref{fig:cifar_comparison}. The corresponding values can be found in Table~\ref{tab:imagenet_results}.}
    \label{fig:imagenet_comparison}
  \end{figure}
From the experiments, we see the trade-off between the two rank decay mechanisms: linear decay is less aggressive, which yield to a better accuracy, whereas logarithmic decay reduce a greater number of FLOPs while still having a decent accuracy.
To further validate \gls{scef}, we run the same experiments on three commonly used base networks. The results are reported in Tab.~\ref{tab:comparison} and Tab.~\ref{tab:imagenet_results} to compare with the corresponding base network and state-of-the-art model reduction techniques.
\begin{table}[h!]
  \begin{center}
    \resizebox{\columnwidth}{!}{\begin{tabular}{|c|c|c|c|c|c|c|}
    \hline
                                  & layers       & rank decay  & Top-1         & Top-5         & params       & (G)FLOPs   \\
    \hline
    \multirow{3}{*}{ResNet-50}    & \gls{conv2d} & None        & 76.47 \%      & 93.21\%       & 25.56M       & 3.80       \\
                                  & \gls{scef}   & Linear      & {\bf 76.61}\% & 93.22\%       & 17.27M       & 2.90       \\
                                  & \gls{scef}   & Logarithmic & 76.46\%       & {\bf 93.24\%} & {\bf 16.64M} & {\bf 2.50} \\
    \hline
    \multirow{3}{*}{DenseNet-121} & \gls{conv2d} & None        & 74.81\%       & 92.32\%       & 79.79M       & 2.83       \\
                                  & \gls{scef}   & Linear      & {\bf 74.85\%} & {\bf 92.61\%} & 72.10M       & 2.81       \\
                                  & \gls{scef}   & Logarithmic & 74.40\%       & 91.89\%       & {\bf 62.92M} & {\bf 2.11} \\
    \hline
    \multirow{3}{*}{HRNet-W18-C}  & \gls{conv2d} & None        & {\bf 77.00}\%       & {\bf 93.50\%}       & 21.30M       & 3.99       \\
                                  & \gls{scef}   & Linear      & 76.17\% & 92.99\% & 9.490M       & 2.55       \\
                                  & \gls{scef}   & Logarithmic & 75.11\%       & 92.47\%       & {\bf 7.05M}  & {\bf 1.27} \\
    \hline
\end{tabular}}
  \end{center}
  \caption{Comparison to the base networks on ImageNet.}
  \label{tab:comparison}
\end{table}
\begin{table}[ht!]
  \begin{center}
    \resizebox{\columnwidth}{!}{\begin{tabular}{lrrrr}\toprule
Network & Top-5 Acc. & Top-1 Acc. & No. param. & GFLOPs\\
\midrule\multicolumn{5}{l}{{\bf (a) \Glsentryshort{scef} vs baseline network}} \\\hline
HRNet-W18-C\textsuperscript{0} \cite{wang2020deep} & {\SI{93.50}{\percent}} & {\SI{77.00}{\percent}} & {\SI{21.30}{\mega\nothing}} & \SI{3.99}{\nothing} \\
DeCEF-ResNet-50 (lin decay)\textsuperscript{1}  & {\SI{93.22}{\percent}} & {\SI{76.61}{\percent}} & {\SI{17.27}{\mega\nothing}} & \SI{2.90}{\nothing} \\
ResNet-50\textsuperscript{2} \cite{He2016b} & {\SI{93.21}{\percent}} & {\SI{76.47}{\percent}} & {\SI{25.56}{\mega\nothing}} & \SI{3.80}{\nothing} \\
DeCEF-ResNet-50 (log decay)\textsuperscript{3}  & {\SI{93.24}{\percent}} & {\SI{76.46}{\percent}} & {\SI{16.64}{\mega\nothing}} & \SI{2.50}{\nothing} \\
DeCEF-HRNet-W18-C (lin decay)\textsuperscript{4}  & {\SI{92.99}{\percent}} & {\SI{76.17}{\percent}} & {\SI{9.49}{\mega\nothing}} & \SI{2.55}{\nothing} \\
DeCEF-HRNet-W18-C (log decay)\textsuperscript{5}  & {\SI{92.47}{\percent}} & {\SI{75.11}{\percent}} & {\SI{7.05}{\mega\nothing}} & \SI{1.27}{\nothing} \\
\midrule\multicolumn{5}{l}{{\bf (b) Related work}} \\\hline
GFP ResNet-50 1\textsuperscript{6} \cite{liu21ab} & {} & {\SI{76.95}{\percent}} & {} & \SI{3.06}{\nothing} \\
Taylor-FO-BN-91\%\textsuperscript{7} \cite{Molchanov_2019_CVPR} & {} & {\SI{76.43}{\percent}} & {\SI{22.60}{\mega\nothing}} & \SI{3.27}{\nothing} \\
GFP ResNet-50 2\textsuperscript{8} \cite{liu21ab} & {} & {\SI{76.42}{\percent}} & {} & \SI{2.04}{\nothing} \\
MetaPruning 0.85 ResNet-50\textsuperscript{9} \cite{liu2019metapruning} & {} & {\SI{76.20}{\percent}} & {} & \SI{3.00}{\nothing} \\
GBN-60\textsuperscript{10} \cite{you2019gate} & {\SI{92.83}{\percent}} & {\SI{76.19}{\percent}} & {\SI{17.42}{\mega\nothing}} & \SI{2.25}{\nothing} \\
ResNet-50 GAL-0.5-joint\textsuperscript{11} \cite{lin2019towards} & {\SI{90.82}{\percent}} & {\SI{76.15}{\percent}} & {\SI{19.31}{\mega\nothing}} & \SI{1.84}{\nothing} \\
ResRep ResNet-50 1\textsuperscript{12} \cite{ding2020lossless} & {\SI{92.90}{\percent}} & {\SI{76.15}{\percent}} & {} & \SI{1.67}{\nothing} \\
SSS-ResNetXt-41\textsuperscript{13} \cite{huang2018data} & {\SI{93.00}{\percent}} & {\SI{75.93}{\percent}} & {\SI{12.40}{\mega\nothing}} & \SI{3.23}{\nothing} \\
SASL\textsuperscript{14} \cite{shi2021} & {\SI{92.82}{\percent}} & {\SI{75.76}{\percent}} & {} & \SI{1.91}{\nothing} \\
AOFP-C1\textsuperscript{15} \cite{ding19a} & {\SI{92.69}{\percent}} & {\SI{75.63}{\percent}} & {} & \SI{2.58}{\nothing} \\
ResRep ResNet-50 2\textsuperscript{16} \cite{ding2020lossless} & {\SI{92.55}{\percent}} & {\SI{75.49}{\percent}} & {} & \SI{1.44}{\nothing} \\
Taylor-FO-BN-81\%\textsuperscript{17} \cite{Molchanov_2019_CVPR} & {} & {\SI{75.48}{\percent}} & {\SI{17.90}{\mega\nothing}} & \SI{2.66}{\nothing} \\
SSS-ResNet-41\textsuperscript{18} \cite{huang2018data} & {\SI{92.61}{\percent}} & {\SI{75.44}{\percent}} & {\SI{25.30}{\mega\nothing}} & \SI{3.47}{\nothing} \\
MetaPruning 0.75 ResNet-50\textsuperscript{19} \cite{liu2019metapruning} & {} & {\SI{75.40}{\percent}} & {} & \SI{2.00}{\nothing} \\
GBN-50\textsuperscript{20} \cite{you2019gate} & {\SI{92.41}{\percent}} & {\SI{75.18}{\percent}} & {\SI{11.91}{\mega\nothing}} & \SI{1.71}{\nothing} \\
SASL*\textsuperscript{21} \cite{shi2021} & {\SI{92.47}{\percent}} & {\SI{75.15}{\percent}} & {} & \SI{1.67}{\nothing} \\
AOFP-C2\textsuperscript{22} \cite{ding19a} & {\SI{92.28}{\percent}} & {\SI{75.11}{\percent}} & {} & \SI{1.66}{\nothing} \\
ResNet-50 FPGM-only 30\%\textsuperscript{23} \cite{he2019filter} & {\SI{92.40}{\percent}} & {\SI{75.03}{\percent}} & {} & \SI{2.23}{\nothing} \\
ResNet-50 HRank 1\textsuperscript{24} \cite{lin2020hrank} & {\SI{92.33}{\percent}} & {\SI{74.98}{\percent}} & {\SI{16.15}{\mega\nothing}} & \SI{2.30}{\nothing} \\
SSS-ResNetXt-38\textsuperscript{25} \cite{huang2018data} & {\SI{92.50}{\percent}} & {\SI{74.98}{\percent}} & {\SI{10.70}{\mega\nothing}} & \SI{2.43}{\nothing} \\
DCP\textsuperscript{26} \cite{zhuang2018discrimination} & {\SI{92.32}{\percent}} & {\SI{74.95}{\percent}} & {\SI{12.41}{\mega\nothing}} & \SI{1.69}{\nothing} \\
DCP\textsuperscript{27} \cite{zhuang2018discrimination} & {\SI{92.32}{\percent}} & {\SI{74.95}{\percent}} & {\SI{12.41}{\mega\nothing}} & \SI{1.69}{\nothing} \\
SFP\textsuperscript{28} \cite{he2018soft} & {\SI{92.06}{\percent}} & {\SI{74.61}{\percent}} & {} & \SI{2.19}{\nothing} \\
SSS-ResNetXt-35-A\textsuperscript{29} \cite{huang2018data} & {\SI{92.17}{\percent}} & {\SI{74.57}{\percent}} & {\SI{10.00}{\mega\nothing}} & \SI{2.07}{\nothing} \\
C-SGD-50\textsuperscript{30} \cite{ding2019centripetal} & {\SI{92.09}{\percent}} & {\SI{74.54}{\percent}} & {} & \SI{1.71}{\nothing} \\
Taylor-FO-BN-72\%\textsuperscript{31} \cite{Molchanov_2019_CVPR} & {} & {\SI{74.50}{\percent}} & {\SI{14.20}{\mega\nothing}} & \SI{2.25}{\nothing} \\
LFPC\textsuperscript{32} \cite{He_2020_CVPR} & {\SI{92.04}{\percent}} & {\SI{74.46}{\percent}} & {} & \SI{1.60}{\nothing} \\
SSS-ResNet-32\textsuperscript{33} \cite{huang2018data} & {\SI{91.91}{\percent}} & {\SI{74.18}{\percent}} & {\SI{18.60}{\mega\nothing}} & \SI{2.82}{\nothing} \\
GFP ResNet-50 3\textsuperscript{34} \cite{liu21ab} & {} & {\SI{73.94}{\percent}} & {} & \SI{1.02}{\nothing} \\
Pruned-90\textsuperscript{35} \cite{Luo2017b} & {\SI{91.60}{\percent}} & {\SI{73.56}{\percent}} & {\SI{23.89}{\mega\nothing}} & \SI{3.58}{\nothing} \\
MetaPruning 0.5 ResNet-50\textsuperscript{36} \cite{liu2019metapruning} & {} & {\SI{73.40}{\percent}} & {} & \SI{1.00}{\nothing} \\
SSS-ResNetXt-35-B\textsuperscript{37} \cite{huang2018data} & {\SI{91.58}{\percent}} & {\SI{73.17}{\percent}} & {\SI{8.50}{\mega\nothing}} & \SI{1.55}{\nothing} \\
\end{tabular}
}
  \end{center}
  \caption{Comparison to state-of-the-art model reduction techniques on ImageNet.} 
  \label{tab:imagenet_results}
\end{table}
\subsection{Limitation}
This work focuses on the model reduction aspect given the observations of the low rank behaviors. The analysis of these behaviors needs to be further explored as a future direction, which brings the limitation that the rules for choosing the hyperparameters are rather heuristic.
\section{Conclusion and future work}
In this paper, we propose a new methodology to observe and analyze the redundancy in a \gls{cnn}. Motivated by our observations of the low rank behaviors in vectorized \gls{conv2d} filters, we present a layer structure \gls{scef} as an alternative parameterization to \gls{conv2d} filters for the purpose of reducing their complexity in terms of trainable parameters and \gls{flops}.
Our experiments have shown that in a convolutional layer with filter size $h\times h$, it is not necessary to have more than $h^2$ eigen-filters given the training strategy in Sec.~\ref{sec:algorithms}.

In terms of the accuracy-to-complexity ratio, it is beneficial to use more coefficients (i.e. output channels) with fewer eigen-filters in \gls{scef} layers.
The \gls{scef} layer is simple to implement in most deep learning frameworks using depthwise separable convolutions with a new training strategy. 
With the deterministic rules for choosing hyperparameters, it is easy to design and reproduce the results.
From our observations, the underlying subspace structure is a commonly shared property among different network architectures and topologies, which provides insights to the design and analysis of \gls{cnn}s.

As future directions, first we will further analyze this low rank structure to improve rank decay functions by designing more sophisticated strategies. For instance, in some network architectures, the effective rank first increases and then quickly decreases with respect to the depth. This is a phenomenon that we would like to study further. Moreover, since the \gls{scef} layer can be implemented by the depthwise separable convolutions with a new training strategy, a second future direction is to modify and train the traditional depthwise separable convolutional layers in well-known networks using \gls{scef} to reduce the model complexity. Finally, during the experiments, we have come up with several hypotheses regarding the low rank behaviors in deep neural networks, which we plan to explore to better understand and interpret a \gls{cnn} from this perspective.

\clearpage


\bibliographystyle{plainnat}

\bibliography{library}

\end{document}


\title{Supplementary Material}

\section{Proof of Lemma 1}
\begin{proof}
    For each input channel $i$, given an additive perturbation matrix $\Delta\mathbf{I}_i$, let $\tilde{\mathbf{I}}_i=\mathbf{I}_i+\Delta\mathbf{I}_i$. Given optimal parameters of kernel $j$ expressed as $\mathbf{w}_{j}^{(i)}=\sum_{k=1}^ra_{k,j}^{(i)}\mathbf{u}_{k}^{(i)}$, which are learned from the training data,
    the output of the convolutional layer is
    \begin{eqnarray*}
        \mathbf{I}_j&=&\sum_{i}\left(\mathbf{I}_i+\Delta\mathbf{I}_i\right)*\mathbf{w}_{j}^{(i)}
        =\sum_{i}\left(\mathbf{I}_i+\Delta\mathbf{I}_i\right)*\sum_{k=1}^ra_{k,j}^{(i)}\mathbf{u}_{k}^{(i)}\\
        &=&\sum_{i}\mathbf{I}_i*\sum_{k=1}^ra_{k,j}^{(i)}\mathbf{u}_{k}+\sum_{i}\Delta\mathbf{I}_i*\sum_{k=1}^ra_{k,j}^{(i)}\mathbf{u}_{k}^{(i)}\\
        &=&\mathbf{I}_j^* + \underbrace{\sum_{i}\Delta\mathbf{I}_i*\sum_{k=1}^ra_{k,j}^{(i)}\mathbf{u}_{k}^{(i)}}_{\text{perturbation term}}
    \end{eqnarray*}
    where $\mathbf{I}_j^*$ denotes the optimal feature map.
    By using the infinity norm to characterize the effect of the perturbation, we have:
    \begin{eqnarray}
        \label{eqa:l2_0}
        \left|\left|\mathbf{I}_j^* - \mathbf{I}_j\right|\right|_{\infty} &=&\left|\left|\sum_{i}\Delta\mathbf{I}_i*\sum_{k=1}^ra_{k,j}^{(i)}\mathbf{u}_{k}^{(i)}\right|\right|_{\infty}\\
        &\leq&  \sum_{i}\left|\left|\Delta\mathbf{I}_i*\sum_{k=1}^ra_{k,j}^{(i)}\mathbf{u}_{k}^{(i)}\right|\right|_{\infty}
    \end{eqnarray}
    From Young's inequality:
    \begin{eqnarray}
        \nonumber
        \text{\eqref{eqa:l2_0}} &\leq & \sum_{i}\left|\left|\Delta\mathbf{I}_i\right|\right|_2\left|\left|\sum_{k=1}^ra_{k,j}^{(i)}\mathbf{u}_{k}^{(i)}\right|\right|_2\\
        \nonumber
        &\leq&  \sum_{i}\left|\left|\Delta\mathbf{I}_i\right|\right|_2\left|\left|\sum_{k=1}^ra_{k,j}^{(i)}\mathbf{u}_{k}^{(i)}\right|\right|_2\\
        \nonumber
        &\leq& \sum_{i}\left|\left|\Delta\mathbf{I}_i\right|\right|_2\sum_{k=1}^r\left| a_{k,j}^{(i)}\right|\left|\left|\mathbf{u}_{k}^{(i)}\right|\right|_2
    \end{eqnarray}
    \begin{eqnarray}
        \label{eqa:l2_1}
        &\leq&  \sum_{i}\left|\left|\Delta\mathbf{I}_i\right|\right|_2\sum_{k=1}^r\left| a_{k,j}^{(i)}\right|\left|\left|\mathbf{u}_{k}^{(i)}\right|\right|_F \\
        \nonumber
        &\leq&  \sum_{i}\left|\left|\Delta\mathbf{I}_i\right|\right|_2\left|\left|\mathbf{a}_j^{(i)}\right|\right|_1\sum_{k=1}^r\left|\left|\mathbf{u}_{k}^{(i)}\right|\right|_F
    \end{eqnarray}
    where $\mathbf{a}_j^{(i)}= \begin{bmatrix} a_{1,j}^{(i)}&\cdots & a_{r,j}^{(i)}\end{bmatrix}^{\text{T}}$ and $\|\cdot\|_F$ denotes the Frobenius norm.
    Let $\bar{\mathbf{u}}_{k}^{(i)}=\text{vect}(\mathbf{u}_{k}^{(i)})$, where $\text{vect}(\cdot)$ denotes the vectorization of a matrix. If $\bar{\mathbf{U}}^{(i)\text{T}}\bar{\mathbf{U}}^{(i)}=\mathbf{I}$, we have $\left|\left|\mathbf{u}_k^{(i)}\right|\right|_F=1$ and hence
    \begin{eqnarray}
        \nonumber
        \text{\eqref{eqa:l2_1}}  &\leq& \sum_{i}r\left|\left|\mathbf{a}_j^{(i)}\right|\right|_1\left|\left|\Delta\mathbf{I}_i\right|\right|_2\leq \sum_{i}rh\left|\left|\mathbf{a}_j^{(i)}\right|\right|_2\left|\left|\Delta\mathbf{I}_i\right|\right|_2
    \end{eqnarray}
    where $h$ is the kernel size.
    If $\left|\left|\mathbf{a}^{(i)}_{j}\right|\right|_2\leq \epsilon$, $\forall i, j$, then
    $$ \left|\left|\sum_{i}\Delta\mathbf{I}_i*\mathbf{w}_{j}^{(i)}\right|\right|_{\infty}\leq \epsilon hr\sum_{i}\left|\left|\Delta\mathbf{I}_i\right|\right|_2$$
\end{proof}




\section{Network compression using \glsentryshort{scef} layers}
One application of \glsentryshort{scef} is to use it as a model reduction technique for a pre-trained network. As discussed in the paper, this is not the main focus of \glsentryshort{scef}. Nevertheless, we propose an algorithm as follows for this type of applications.
\begin{algorithm}(\glsentryshort{scef}C-basenet)
    \label{alg:usecase2}

    \begin{itemize}
        \item[] {\bf Step 1:} Analysis described in Procedure 1.
        \item[] {\bf Step 2:} For each layer, let $\bar{\mathbf{u}}_{k}^{(i)}$ be the columns of $\bar{\mathbf{U}}^{(i)}$. Approximate $\mathbf{w}_j^{(i)}$ by
              $  \mathbf{w}_j^{(i)}\approx \sum^{r_i}_{k=1}a_{k,j}^{(i)}\mathbf{u}_k^{(i)}$,
              where $\mathbf{u}_{k}^{(i)}$ is obtained by reshaping $\bar{\mathbf{u}}_{k}^{(i)}$ into a $h\times h$ matrix.
        \item[] {\bf Step 3 (optional):} Network fine-tuning by freezing the eigen-filters and training the other trainable parameters.
    \end{itemize}
\end{algorithm}

\section{Demonstration of singular values computed using {\bf Procedure 1}}

Figure \ref{fig:densenet121} shows the density histogram of singular values computed using Eq.~[1] in the paper. The histogram is calculated from all input channels in each convolutional layer with $K>1$. The maximum ranks of the layers in these example networks are $\min(K, c_{\text{out}})=K$, where $K=9$.
Similar low rank behaviors can be observed in Fig.~\ref{fig:network_ranks}.
\begin{figure*}[h!]
    \begin{center}
    \includegraphics[width=1.00\linewidth]{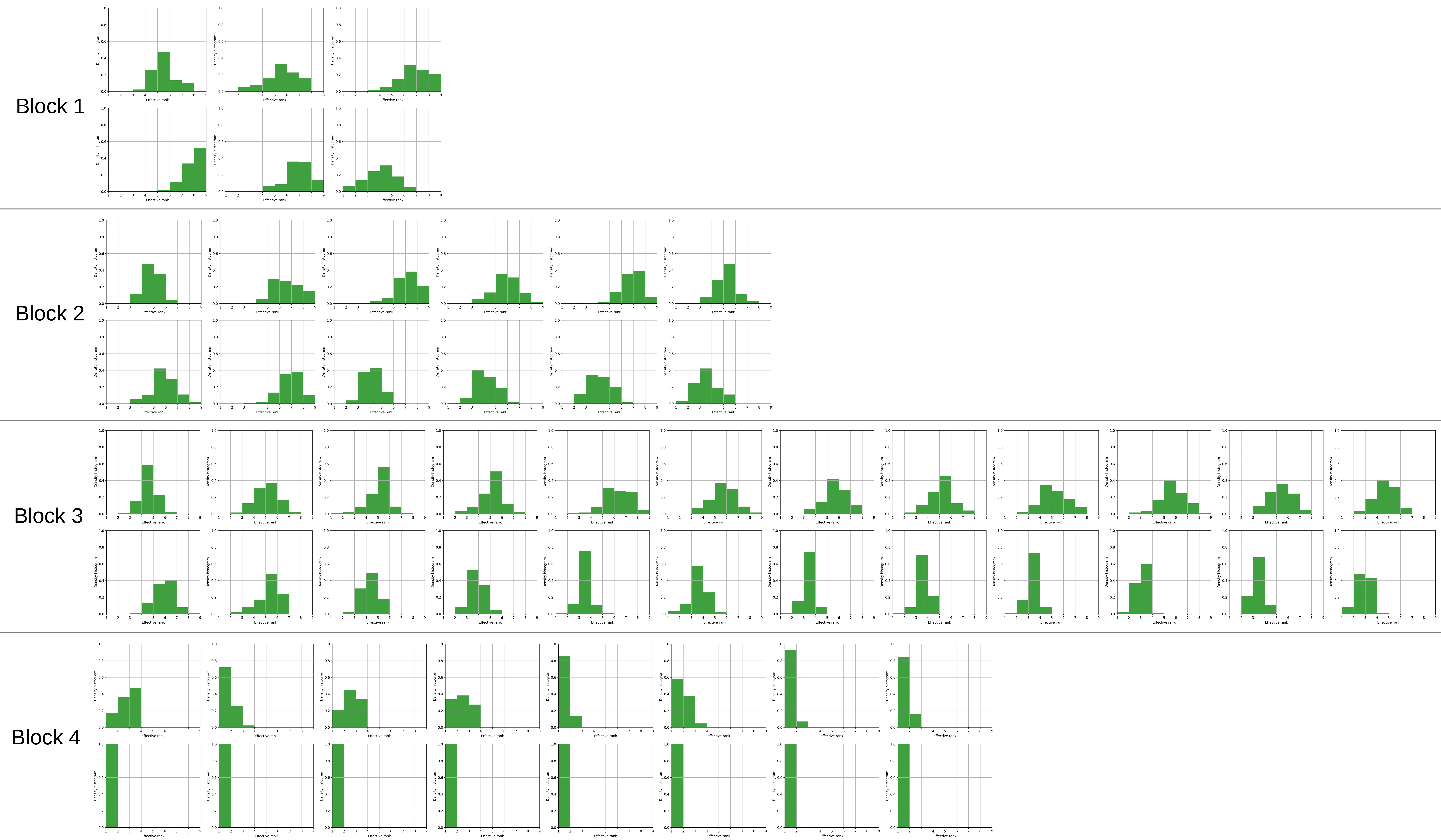}
    \caption{The density histogram (y-axis$\in [0,1]$) of the effective rank (x-axis$\in [1,9]$) estimated using Eq.~[2] in Procedure 1 for DenseNet-121 trained on ImageNet. The statistics are computed over all input channels $i$ in that layer. We see the decreasing trend of the effective ranks with respect to the depth of the network.}
    \label{fig:densenet121}
    \end{center}
\end{figure*}
\begin{figure}[hb!]
    \centering
    \includegraphics[width=0.8\linewidth]{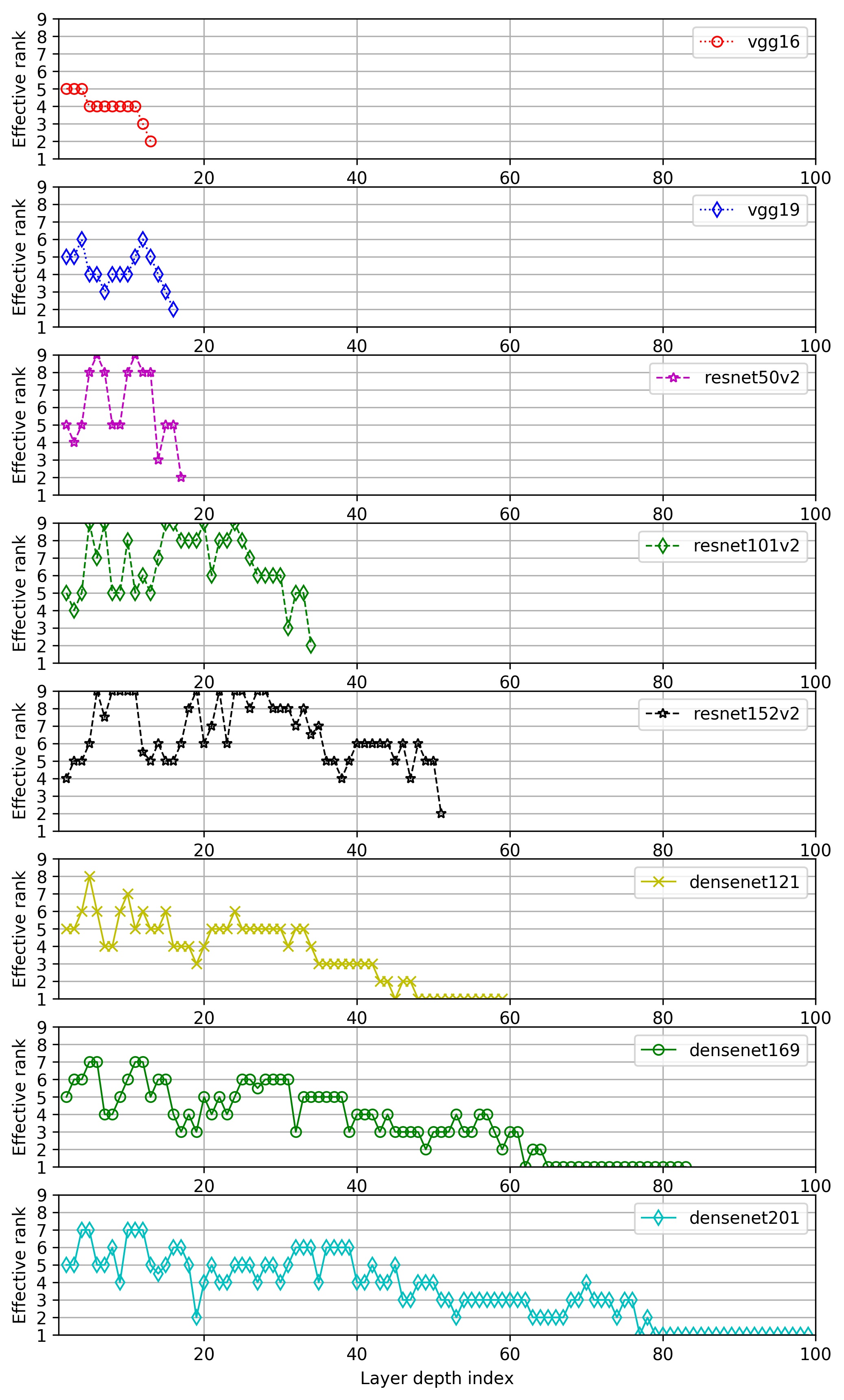}
    \caption{Effective rank (cf.~Eq.~[3] in the paper) versus layer depth. In these networks, we observe decreasing trend of the effective ranks when a network goes deeper. In this figure, we show this effect in the networks VGG, ResNet and DenseNet.}
    \label{fig:network_ranks}
\end{figure}

\section{Convergence video}
In this supplementary material, we include videos (in the file called ``effective\_rank\_converge\_video.zip'') to show some examples of how the effective ranks (cf.~Eq.[2] in the paper) in each layer converge over training epochs. The title of each video indicates the layer index, i.e. the larger the index, the deeper the layer is. The network and data used here are the DenseNet-121 and ImageNet, respectively.

In the video, the leftmost rectangular box shows the singular values computed from all the output channels (filters) for each input channel. Each row in this figure contains the singular values for one input channel. The image in the middle is the histogram density of the effective rank. Each frame in this video shows the singular values and effective ranks computed from one epoch. Finally, the image on the right shows the convergence of this effective rank over raining epochs.
Note that due to the limit on the file size, we only show the convergence for every fourth layer.

\section{Summary of results}

In this section, we present more detailed results from our experiments.
Result comparisons can be found in Tab.~\ref{tab:comparision} and Tab.~\ref{tab:comparision_imagenet} for CIFAR-10 and ImageNet, respectively. These results are also summarized as ball charts in Fig.~\ref{fig:cifar_comparison} for CIFAR-10 and in Fig.~\ref{fig:imagenet_comparison_top5} and Fig.~\ref{fig:imagenet_comparison_top1} for ImageNet.

\begin{figure}[ht!]
    \centering
    \includegraphics[width=1\linewidth]{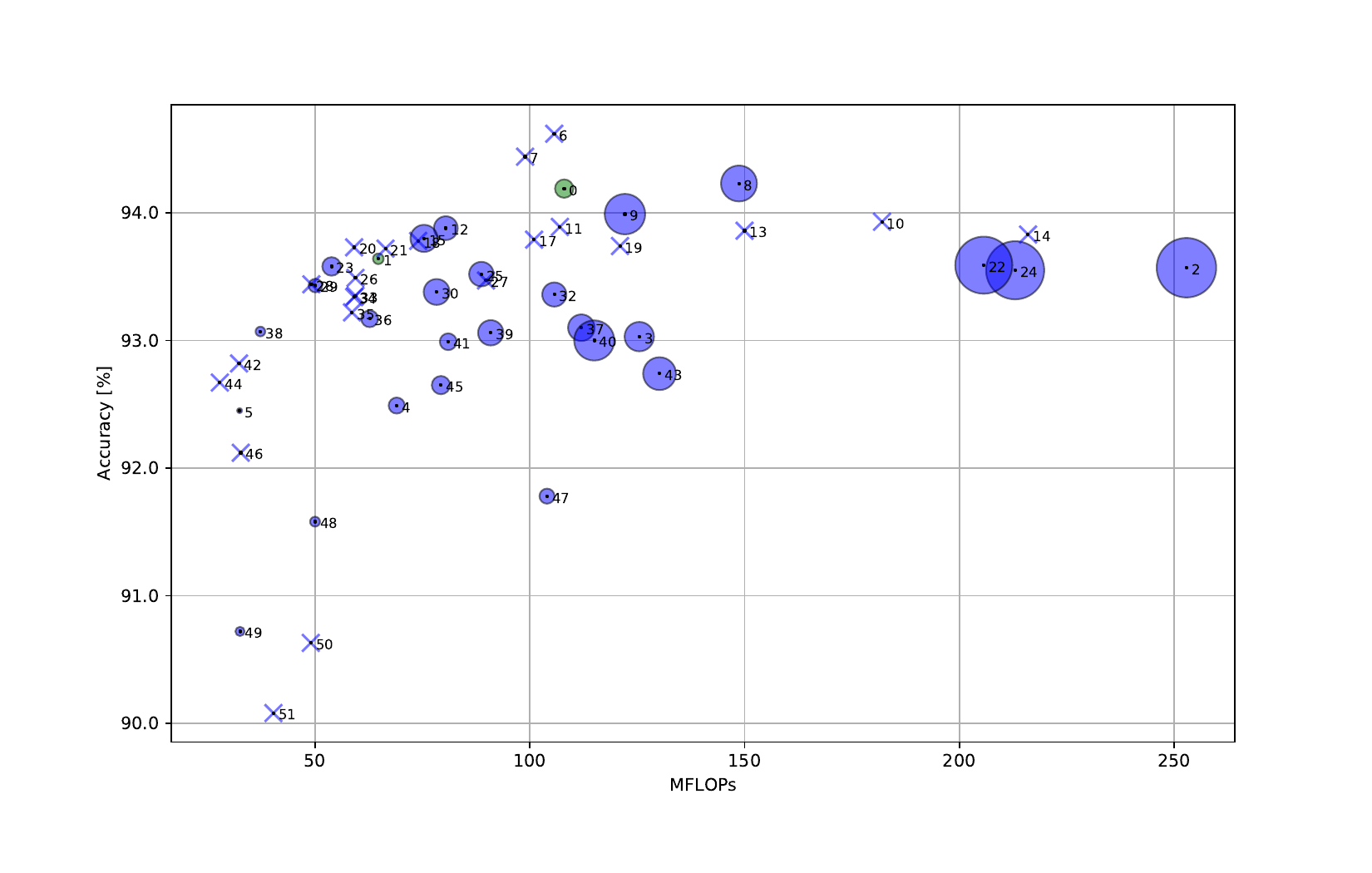}
    \caption{Accuracy (y-axis) versus FLOPs (x-axis) and number of trainable parameters (indicated by the radius of each ball) for CIFAR-10. Algorithms that have not reported their number of parameters are denoted as crosses (x) in the chart. More details can be found in Table \ref{tab:comparision}. The number in the ball chart is the id of the algorithm that is indicated as the superscript in the table.}
    \label{fig:cifar_comparison}
\end{figure}


\begin{figure}[ht!]
    \centering
    \includegraphics[width=1\linewidth]{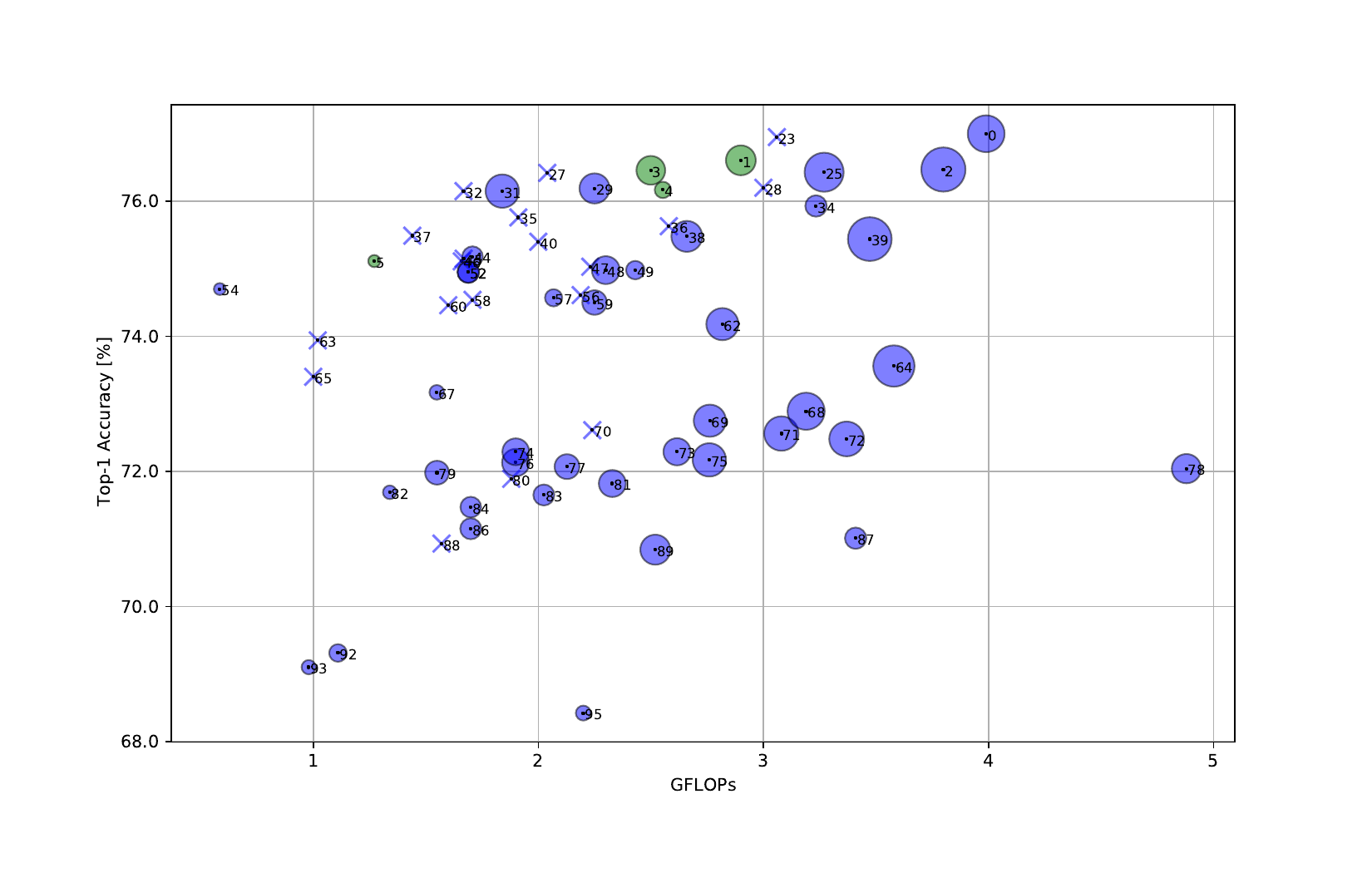}
    \caption{Top-1 accuracy (y-axis) versus FLOPs (x-axis) and number of trainable parameters (indicated by the radius of each ball) for ImageNet. Algorithms that have not reported their number of parameters are denoted as crosses (x) in the chart. More details can be found in Table \ref{tab:comparision_imagenet}. The number in the ball chart is the id of the algorithm that is indicated as the superscript in the table.}
    \label{fig:imagenet_comparison_top1}
\end{figure}

\begin{figure}[ht!]
    \centering
    \includegraphics[width=1\linewidth]{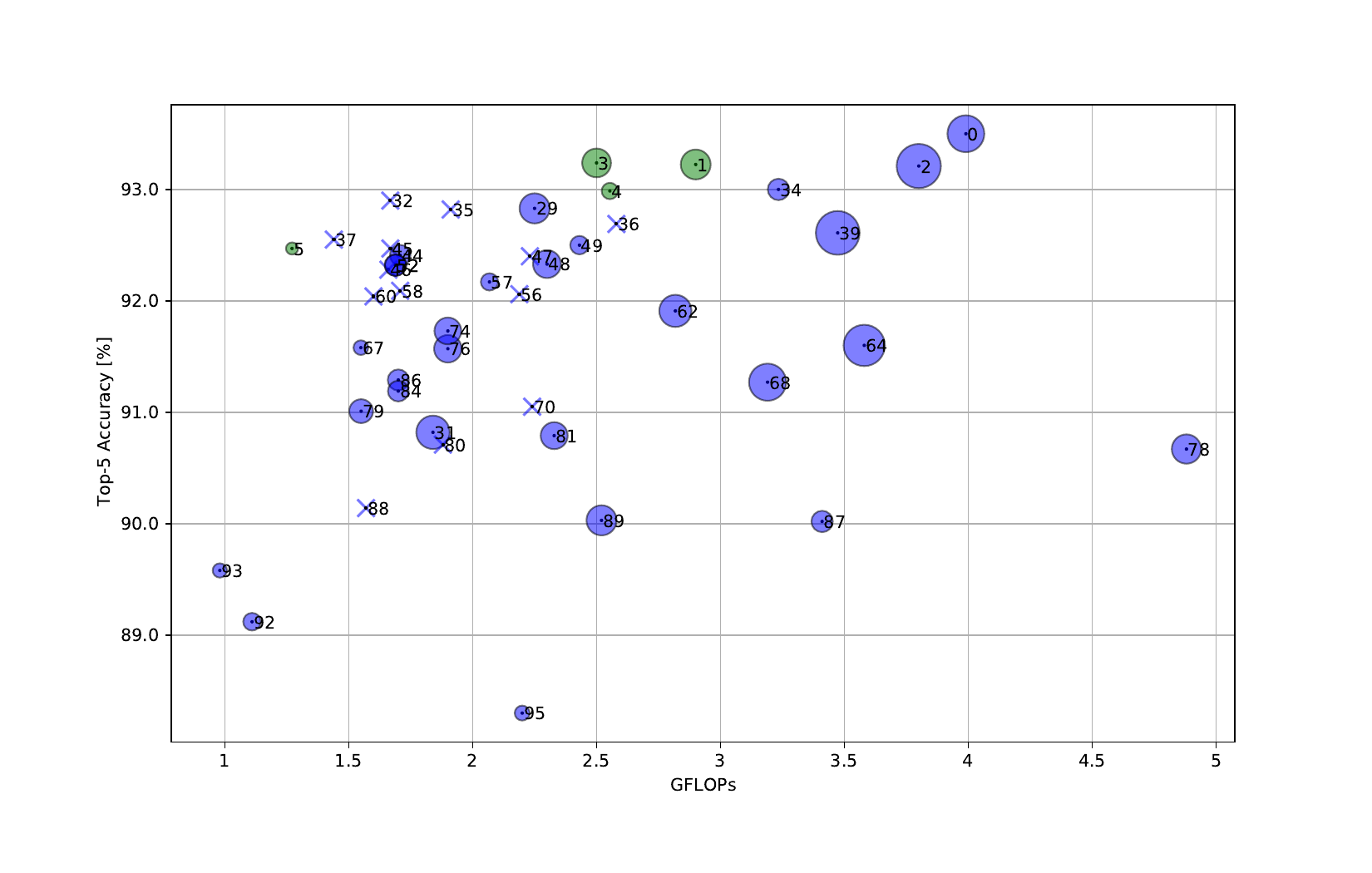}
    \caption{Top-5 accuracy (y-axis) versus FLOPs (x-axis) and number of trainable parameters (indicated by the radius of each ball) for ImageNet. Algorithms that have not reported their number of parameters are denoted as crosses (x) in the chart. More details can be found in Table \ref{tab:comparision_imagenet}. The number in the ball chart is the id of the algorithm that is indicated as the superscript in the table.}
    \label{fig:imagenet_comparison_top5}
\end{figure}


\begin{table}[h]
    \centering
    \caption{Comparison to related work on CIFAR-10 dataset. Part (a): our networks refer to \glsentryshort{scef} with number of filters for each bock in the parenthesis ($c_{\text{out},1}$, $c_{\text{out},2}$, $c_{\text{out},3}$), where $c_{\text{out},b}$ is the number of filters in the layers from ResNet-32 block $b$. We train the baseline network ResNet-32 for comparison. Part (b): state-of-the-art results reported in literatures.}
    \label{tab:comparision}
    \resizebox{0.7\columnwidth}{!}{\begin{tabular}{lrlrr}\toprule
Network & Acc. & Std. & No. param. & MFLOPs\\
\midrule\multicolumn{5}{l}{{\bf (a) \Glsentryshort{scef} vs baseline network}} \\\hline
DeCEF-ResNet-32 (32, 64, 128)\textsuperscript{0}  & {\SI{94.19}{\percent}} & (\SI{0.18}{\percent}) & {\SI{533.00}{\kilo\nothing}} & \SI{108.00}{\nothing} \\
DeCEF-ResNet-32 (24, 48, 96)\textsuperscript{1}  & {\SI{93.64}{\percent}} & (\SI{0.16}{\percent}) & {\SI{311.00}{\kilo\nothing}} & \SI{64.72}{\nothing} \\
ResNet-110\textsuperscript{2} \cite{He2016} & {\SI{93.57}{\percent}} &  & {\SI{1.72}{\mega\nothing}} & \SI{252.89}{\nothing} \\
ResNet-56\textsuperscript{3} \cite{He2016} & {\SI{93.03}{\percent}} &  & {\SI{850.00}{\kilo\nothing}} & \SI{125.49}{\nothing} \\
ResNet-32\textsuperscript{4} \cite{He2016} & {\SI{92.49}{\percent}} &  & {\SI{467.00}{\kilo\nothing}} & \SI{69.00}{\nothing} \\
DeCEF-ResNet-32 (16, 32, 64)\textsuperscript{5}  & {\SI{92.45}{\percent}} & (\SI{0.17}{\percent}) & {\SI{148.00}{\kilo\nothing}} & \SI{32.42}{\nothing} \\
\midrule\multicolumn{5}{l}{{\bf (b) Related work}} \\\hline
ResRep ResNet-110\textsuperscript{6} \cite{ding2020lossless} & {\SI{94.62}{\percent}} &  & {} & \SI{105.68}{\nothing} \\
C-SGD-5/8 ResNet-110\textsuperscript{7} \cite{ding2019centripetal} & {\SI{94.44}{\percent}} &  & {} & \SI{98.91}{\nothing} \\
HRank ResNet-110 1\textsuperscript{8} \cite{lin2020hrank} & {\SI{94.23}{\percent}} &  & {\SI{1.04}{\mega\nothing}} & \SI{148.70}{\nothing} \\
SASL ResNet-110\textsuperscript{9} \cite{shi2021} & {\SI{93.99}{\percent}} &  & {\SI{1.17}{\mega\nothing}} & \SI{122.15}{\nothing} \\
SFP ResNet-110 20\%\textsuperscript{10} \cite{he2018soft} & {\SI{93.93}{\percent}} &  & {} & \SI{182.00}{\nothing} \\
SFP ResNet-56 10\%\textsuperscript{11} \cite{he2018soft} & {\SI{93.89}{\percent}} &  & {} & \SI{107.00}{\nothing} \\
SASL ResNet-56\textsuperscript{12} \cite{shi2021} & {\SI{93.88}{\percent}} &  & {\SI{689.35}{\kilo\nothing}} & \SI{80.44}{\nothing} \\
SFP ResNet-110 30\%\textsuperscript{13} \cite{he2018soft} & {\SI{93.86}{\percent}} &  & {} & \SI{150.00}{\nothing} \\
SFP ResNet-110 10\%\textsuperscript{14} \cite{he2018soft} & {\SI{93.83}{\percent}} &  & {} & \SI{216.00}{\nothing} \\
SASL* ResNet-110\textsuperscript{15} \cite{shi2021} & {\SI{93.80}{\percent}} &  & {\SI{786.04}{\kilo\nothing}} & \SI{75.36}{\nothing} \\
ShaResNet-164\textsuperscript{16} \cite{Boulch2018} & {\SI{93.80}{\percent}} &  & {\SI{930.00}{\kilo\nothing}} &  \\
LFPC ResNet-110\textsuperscript{17} \cite{He_2020_CVPR} & {\SI{93.79}{\percent}} &  & {} & \SI{101.00}{\nothing} \\
SFP ResNet-56 30\%\textsuperscript{18} \cite{he2018soft} & {\SI{93.78}{\percent}} &  & {} & \SI{74.00}{\nothing} \\
FPGM-only 40\% ResNet-110\textsuperscript{19} \cite{he2019filter} & {\SI{93.74}{\percent}} &  & {} & \SI{121.00}{\nothing} \\
ResRep ResNet-56 1\textsuperscript{20} \cite{ding2020lossless} & {\SI{93.73}{\percent}} &  & {} & \SI{59.09}{\nothing} \\
LFPC ResNet-56 1\textsuperscript{21} \cite{He_2020_CVPR} & {\SI{93.72}{\percent}} &  & {} & \SI{66.40}{\nothing} \\
GAL-0.1 ResNet-110\textsuperscript{22} \cite{lin2019towards} & {\SI{93.59}{\percent}} &  & {\SI{1.65}{\mega\nothing}} & \SI{205.70}{\nothing} \\
SASL* ResNet-56\textsuperscript{23} \cite{shi2021} & {\SI{93.58}{\percent}} &  & {\SI{538.90}{\kilo\nothing}} & \SI{53.84}{\nothing} \\
ResNet-110-pruned-A\textsuperscript{24} \cite{Li2016} & {\SI{93.55}{\percent}} &  & {\SI{1.68}{\mega\nothing}} & \SI{213.00}{\nothing} \\
HRank ResNet-56 1\textsuperscript{25} \cite{lin2020hrank} & {\SI{93.52}{\percent}} &  & {\SI{710.00}{\kilo\nothing}} & \SI{88.72}{\nothing} \\
FPGM-only 40\% ResNet-56\textsuperscript{26} \cite{he2019filter} & {\SI{93.49}{\percent}} &  & {} & \SI{59.40}{\nothing} \\
SFP ResNet-56 20\%\textsuperscript{27} \cite{he2018soft} & {\SI{93.47}{\percent}} &  & {} & \SI{89.80}{\nothing} \\
C-SGD-5/8 ResNet-56\textsuperscript{28} \cite{ding2019centripetal} & {\SI{93.44}{\percent}} &  & {} & \SI{49.13}{\nothing} \\
GBN-40\textsuperscript{29} \cite{you2019gate} & {\SI{93.43}{\percent}} &  & {\SI{395.25}{\kilo\nothing}} & \SI{50.07}{\nothing} \\
GAL-0.6 ResNet-56\textsuperscript{30} \cite{lin2019towards} & {\SI{93.38}{\percent}} &  & {\SI{750.00}{\kilo\nothing}} & \SI{78.30}{\nothing} \\
NISP-110\textsuperscript{31} \cite{Yu2018} & {\SI{93.38}{\percent}} &  & {\SI{976.10}{\kilo\nothing}} &  \\
HRank ResNet-110 2\textsuperscript{32} \cite{lin2020hrank} & {\SI{93.36}{\percent}} &  & {\SI{700.00}{\kilo\nothing}} & \SI{105.70}{\nothing} \\
SFP ResNet-56 40\%\textsuperscript{33} \cite{he2018soft} & {\SI{93.35}{\percent}} &  & {} & \SI{59.40}{\nothing} \\
LFPC ResNet-56 2\textsuperscript{34} \cite{He_2020_CVPR} & {\SI{93.34}{\percent}} &  & {} & \SI{59.10}{\nothing} \\
SFP ResNet-32 10\%\textsuperscript{35} \cite{he2018soft} & {\SI{93.22}{\percent}} &  & {} & \SI{58.60}{\nothing} \\
HRank ResNet-56 2\textsuperscript{36} \cite{lin2020hrank} & {\SI{93.17}{\percent}} &  & {\SI{490.00}{\kilo\nothing}} & \SI{62.72}{\nothing} \\
ResNet-56-pruned-A\textsuperscript{37} \cite{Li2016} & {\SI{93.10}{\percent}} &  & {\SI{770.10}{\kilo\nothing}} & \SI{112.00}{\nothing} \\
GBN-30\textsuperscript{38} \cite{you2019gate} & {\SI{93.07}{\percent}} &  & {\SI{283.05}{\kilo\nothing}} & \SI{37.27}{\nothing} \\
ResNet-56-pruned-B\textsuperscript{39} \cite{Li2016} & {\SI{93.06}{\percent}} &  & {\SI{733.55}{\kilo\nothing}} & \SI{90.90}{\nothing} \\
ResNet-110-pruned-B\textsuperscript{40} \cite{Li2016} & {\SI{93.00}{\percent}} &  & {\SI{1.16}{\mega\nothing}} & \SI{115.00}{\nothing} \\
NISP-56\textsuperscript{41} \cite{Yu2018} & {\SI{92.99}{\percent}} &  & {\SI{487.90}{\kilo\nothing}} & \SI{81.00}{\nothing} \\
FPGM-mix 40\% ResNet-32\textsuperscript{42} \cite{he2019filter} & {\SI{92.82}{\percent}} &  & {} & \SI{32.30}{\nothing} \\
GAL-0.5 ResNet-110\textsuperscript{43} \cite{lin2019towards} & {\SI{92.74}{\percent}} &  & {\SI{950.00}{\kilo\nothing}} & \SI{130.20}{\nothing} \\
ResRep ResNet-56 2\textsuperscript{44} \cite{ding2020lossless} & {\SI{92.67}{\percent}} &  & {} & \SI{27.82}{\nothing} \\
HRank ResNet-110 3\textsuperscript{45} \cite{lin2020hrank} & {\SI{92.65}{\percent}} &  & {\SI{530.00}{\kilo\nothing}} & \SI{79.30}{\nothing} \\
LFPC ResNet-32\textsuperscript{46} \cite{He_2020_CVPR} & {\SI{92.12}{\percent}} &  & {} & \SI{32.70}{\nothing} \\
nin-c3-lr\textsuperscript{47} \cite{Ioannou2015} & {\SI{91.78}{\percent}} &  & {\SI{438.00}{\kilo\nothing}} & \SI{104.00}{\nothing} \\
GAL-0.8 ResNet-56\textsuperscript{48} \cite{lin2019towards} & {\SI{91.58}{\percent}} &  & {\SI{290.00}{\kilo\nothing}} & \SI{49.99}{\nothing} \\
HRank ResNet-56 3\textsuperscript{49} \cite{lin2020hrank} & {\SI{90.72}{\percent}} &  & {\SI{270.00}{\kilo\nothing}} & \SI{32.52}{\nothing} \\
SFP ResNet-32 20\%\textsuperscript{50} \cite{he2018soft} & {\SI{90.63}{\percent}} &  & {} & \SI{49.00}{\nothing} \\
SFP ResNet-32 30\%\textsuperscript{51} \cite{he2018soft} & {\SI{90.08}{\percent}} &  & {} & \SI{40.30}{\nothing} \\
\end{tabular}
}
\end{table}

\begin{table*}[t]
    \footnotesize
    \centering
    \caption{Comparison to state-of-the-art techniques on the dataset ImageNet (ILSVRC-2012). ``\glsentryshort{scef} (linear decay)'' and ``\glsentryshort{scef} (log decay)'' refer to training a network using Algorithm 1 with linear and log decay, respectively, and ``\glsentryshort{scef}C'' refers to the network compression using Algorithm \ref{alg:usecase2}. *Official Tensorflow\cite{tensorflow2015-whitepaper} implementation.}
    \label{tab:comparision_imagenet}
    \resizebox{0.55\columnwidth}{!}{ \begin{tabular}{lrlrlrr}\toprule
Network & Top-5 Acc. & Std. & Top-1 Acc. & Std. & No. param. & GFLOPs\\
\midrule\multicolumn{7}{l}{{\bf (a) \Glsentryshort{scef} vs baseline network}} \\\hline
HRNet-W18-C\textsuperscript{0} \cite{wang2020deep} & {\SI{93.50}{\percent}} &  & {\SI{77.00}{\percent}} &  & {\SI{21.30}{\mega\nothing}} & \SI{3.99}{\nothing} \\
DeCEF-ResNet-50 (lin decay)\textsuperscript{1}  & {\SI{93.22}{\percent}} & (\SI{0.07}{\percent}) & {\SI{76.61}{\percent}} & (\SI{0.06}{\percent}) & {\SI{17.27}{\mega\nothing}} & \SI{2.90}{\nothing} \\
ResNet-50\textsuperscript{2} \cite{He2016b} & {\SI{93.21}{\percent}} &  & {\SI{76.47}{\percent}} &  & {\SI{25.56}{\mega\nothing}} & \SI{3.80}{\nothing} \\
DeCEF-ResNet-50 (log decay)\textsuperscript{3}  & {\SI{93.24}{\percent}} & (\SI{0.05}{\percent}) & {\SI{76.46}{\percent}} & (\SI{0.05}{\percent}) & {\SI{16.64}{\mega\nothing}} & \SI{2.50}{\nothing} \\
DeCEF-HRNet-W18-C (lin decay)\textsuperscript{4}  & {\SI{92.99}{\percent}} &  & {\SI{76.17}{\percent}} &  & {\SI{9.49}{\mega\nothing}} & \SI{2.55}{\nothing} \\
DeCEF-HRNet-W18-C (log decay)\textsuperscript{5}  & {\SI{92.47}{\percent}} &  & {\SI{75.11}{\percent}} &  & {\SI{7.05}{\mega\nothing}} & \SI{1.27}{\nothing} \\
\midrule\multicolumn{7}{l}{{\bf (b) Related work}} \\\hline
EfficientNet-B7\textsuperscript{6} \cite{tan2019efficientnet} & {\SI{96.84}{\percent}} &  & {\SI{84.43}{\percent}} &  & {\SI{64.10}{\mega\nothing}} &  \\
EfficientNet-B6\textsuperscript{7} \cite{tan2019efficientnet} & {\SI{96.90}{\percent}} &  & {\SI{84.08}{\percent}} &  & {\SI{41.00}{\mega\nothing}} &  \\
EfficientNet-B5\textsuperscript{8} \cite{tan2019efficientnet} & {\SI{96.71}{\percent}} &  & {\SI{83.70}{\percent}} &  & {\SI{28.50}{\mega\nothing}} &  \\
EfficientNet-B4\textsuperscript{9} \cite{tan2019efficientnet} & {\SI{96.26}{\percent}} &  & {\SI{82.96}{\percent}} &  & {\SI{17.70}{\mega\nothing}} &  \\
NASNetLarge\textsuperscript{10} \cite{zoph2018learning} & {\SI{96.00}{\percent}} &  & {\SI{82.50}{\percent}} &  & {\SI{84.90}{\mega\nothing}} &  \\
EfficientNet-B3\textsuperscript{11} \cite{tan2019efficientnet} & {\SI{95.68}{\percent}} &  & {\SI{81.58}{\percent}} &  & {\SI{10.80}{\mega\nothing}} &  \\
InceptionResNetV2\textsuperscript{12} \cite{Szegedy2017} & {\SI{95.25}{\percent}} &  & {\SI{80.26}{\percent}} &  & {\SI{54.30}{\mega\nothing}} &  \\
EfficientNet-B2\textsuperscript{13} \cite{tan2019efficientnet} & {\SI{94.95}{\percent}} &  & {\SI{80.18}{\percent}} &  & {\SI{7.80}{\mega\nothing}} &  \\
EfficientNet-B1\textsuperscript{14} \cite{tan2019efficientnet} & {\SI{94.45}{\percent}} &  & {\SI{79.13}{\percent}} &  & {\SI{6.60}{\mega\nothing}} &  \\
Xception\textsuperscript{15} \cite{Chollet2017} & {\SI{94.50}{\percent}} &  & {\SI{79.00}{\percent}} &  & {\SI{22.86}{\mega\nothing}} &  \\
ResNet152V2\textsuperscript{16} \cite{He2016b} & {\SI{94.16}{\percent}} &  & {\SI{78.03}{\percent}} &  & {\SI{58.30}{\mega\nothing}} &  \\
InceptionV3\textsuperscript{17} \cite{szegedy2016rethinking} & {\SI{93.72}{\percent}} &  & {\SI{77.90}{\percent}} &  & {\SI{21.80}{\mega\nothing}} &  \\
ShaResNet-152\textsuperscript{18} \cite{Boulch2018} & {\SI{93.86}{\percent}} &  & {\SI{77.77}{\percent}} &  & {\SI{36.80}{\mega\nothing}} &  \\
DenseNet201\textsuperscript{19} \cite{huang2016densely} & {\SI{93.62}{\percent}} &  & {\SI{77.32}{\percent}} &  & {\SI{18.30}{\mega\nothing}} &  \\
ResNet101V2\textsuperscript{20} \cite{He2016b} & {\SI{93.82}{\percent}} &  & {\SI{77.23}{\percent}} &  & {\SI{42.60}{\mega\nothing}} &  \\
EfficientNet-B0\textsuperscript{21} \cite{tan2019efficientnet} & {\SI{93.49}{\percent}} &  & {\SI{77.19}{\percent}} &  & {\SI{4.00}{\mega\nothing}} &  \\
ShaResNet-101\textsuperscript{22} \cite{Boulch2018} & {\SI{93.45}{\percent}} &  & {\SI{77.09}{\percent}} &  & {\SI{29.40}{\mega\nothing}} &  \\
GFP ResNet-50 1\textsuperscript{23} \cite{liu21ab} & {} &  & {\SI{76.95}{\percent}} &  & {} & \SI{3.06}{\nothing} \\
ResNet152\textsuperscript{24} \cite{He2016} & {\SI{93.12}{\percent}} &  & {\SI{76.60}{\percent}} &  & {\SI{58.40}{\mega\nothing}} &  \\
Taylor-FO-BN-91\%\textsuperscript{25} \cite{Molchanov_2019_CVPR} & {} &  & {\SI{76.43}{\percent}} &  & {\SI{22.60}{\mega\nothing}} & \SI{3.27}{\nothing} \\
ResNet101\textsuperscript{26} \cite{He2016} & {\SI{92.79}{\percent}} &  & {\SI{76.42}{\percent}} &  & {\SI{42.70}{\mega\nothing}} &  \\
GFP ResNet-50 2\textsuperscript{27} \cite{liu21ab} & {} &  & {\SI{76.42}{\percent}} &  & {} & \SI{2.04}{\nothing} \\
MetaPruning 0.85 ResNet-50\textsuperscript{28} \cite{liu2019metapruning} & {} &  & {\SI{76.20}{\percent}} &  & {} & \SI{3.00}{\nothing} \\
GBN-60\textsuperscript{29} \cite{you2019gate} & {\SI{92.83}{\percent}} &  & {\SI{76.19}{\percent}} &  & {\SI{17.42}{\mega\nothing}} & \SI{2.25}{\nothing} \\
DenseNet169\textsuperscript{30} \cite{huang2016densely} & {\SI{93.18}{\percent}} &  & {\SI{76.18}{\percent}} &  & {\SI{12.60}{\mega\nothing}} &  \\
ResNet-50 GAL-0.5-joint\textsuperscript{31} \cite{lin2019towards} & {\SI{90.82}{\percent}} &  & {\SI{76.15}{\percent}} &  & {\SI{19.31}{\mega\nothing}} & \SI{1.84}{\nothing} \\
ResRep ResNet-50 1\textsuperscript{32} \cite{ding2020lossless} & {\SI{92.90}{\percent}} &  & {\SI{76.15}{\percent}} &  & {} & \SI{1.67}{\nothing} \\
ResNet50V2\textsuperscript{33} \cite{He2016b} & {\SI{93.03}{\percent}} &  & {\SI{75.96}{\percent}} &  & {\SI{23.60}{\mega\nothing}} &  \\
SSS-ResNetXt-41\textsuperscript{34} \cite{huang2018data} & {\SI{93.00}{\percent}} &  & {\SI{75.93}{\percent}} &  & {\SI{12.40}{\mega\nothing}} & \SI{3.23}{\nothing} \\
SASL\textsuperscript{35} \cite{shi2021} & {\SI{92.82}{\percent}} &  & {\SI{75.76}{\percent}} &  & {} & \SI{1.91}{\nothing} \\
AOFP-C1\textsuperscript{36} \cite{ding19a} & {\SI{92.69}{\percent}} &  & {\SI{75.63}{\percent}} &  & {} & \SI{2.58}{\nothing} \\
ResRep ResNet-50 2\textsuperscript{37} \cite{ding2020lossless} & {\SI{92.55}{\percent}} &  & {\SI{75.49}{\percent}} &  & {} & \SI{1.44}{\nothing} \\
Taylor-FO-BN-81\%\textsuperscript{38} \cite{Molchanov_2019_CVPR} & {} &  & {\SI{75.48}{\percent}} &  & {\SI{17.90}{\mega\nothing}} & \SI{2.66}{\nothing} \\
SSS-ResNet-41\textsuperscript{39} \cite{huang2018data} & {\SI{92.61}{\percent}} &  & {\SI{75.44}{\percent}} &  & {\SI{25.30}{\mega\nothing}} & \SI{3.47}{\nothing} \\
MetaPruning 0.75 ResNet-50\textsuperscript{40} \cite{liu2019metapruning} & {} &  & {\SI{75.40}{\percent}} &  & {} & \SI{2.00}{\nothing} \\
ShaResNet-50\textsuperscript{41} \cite{Boulch2018} & {\SI{92.59}{\percent}} &  & {\SI{75.39}{\percent}} &  & {\SI{20.50}{\mega\nothing}} &  \\
MobileNetV2(alpha=1.4)\textsuperscript{42} \cite{Sandler2018} & {\SI{92.42}{\percent}} &  & {\SI{75.23}{\percent}} &  & {\SI{4.40}{\mega\nothing}} &  \\
ResNet-50 Variational\textsuperscript{43} \cite{zhao2019variational} & {\SI{92.10}{\percent}} &  & {\SI{75.20}{\percent}} &  & {\SI{15.30}{\mega\nothing}} &  \\
GBN-50\textsuperscript{44} \cite{you2019gate} & {\SI{92.41}{\percent}} &  & {\SI{75.18}{\percent}} &  & {\SI{11.91}{\mega\nothing}} & \SI{1.71}{\nothing} \\
SASL*\textsuperscript{45} \cite{shi2021} & {\SI{92.47}{\percent}} &  & {\SI{75.15}{\percent}} &  & {} & \SI{1.67}{\nothing} \\
AOFP-C2\textsuperscript{46} \cite{ding19a} & {\SI{92.28}{\percent}} &  & {\SI{75.11}{\percent}} &  & {} & \SI{1.66}{\nothing} \\
ResNet-50 FPGM-only 30\%\textsuperscript{47} \cite{he2019filter} & {\SI{92.40}{\percent}} &  & {\SI{75.03}{\percent}} &  & {} & \SI{2.23}{\nothing} \\
ResNet-50 HRank 1\textsuperscript{48} \cite{lin2020hrank} & {\SI{92.33}{\percent}} &  & {\SI{74.98}{\percent}} &  & {\SI{16.15}{\mega\nothing}} & \SI{2.30}{\nothing} \\
SSS-ResNetXt-38\textsuperscript{49} \cite{huang2018data} & {\SI{92.50}{\percent}} &  & {\SI{74.98}{\percent}} &  & {\SI{10.70}{\mega\nothing}} & \SI{2.43}{\nothing} \\
DenseNet121\textsuperscript{50} \cite{huang2016densely} & {\SI{92.26}{\percent}} &  & {\SI{74.97}{\percent}} &  & {\SI{7.00}{\mega\nothing}} &  \\
DCP\textsuperscript{51} \cite{zhuang2018discrimination} & {\SI{92.32}{\percent}} &  & {\SI{74.95}{\percent}} &  & {\SI{12.41}{\mega\nothing}} & \SI{1.69}{\nothing} \\
DCP\textsuperscript{52} \cite{zhuang2018discrimination} & {\SI{92.32}{\percent}} &  & {\SI{74.95}{\percent}} &  & {\SI{12.41}{\mega\nothing}} & \SI{1.69}{\nothing} \\
ResNet50\textsuperscript{53} \cite{He2016} & {\SI{92.06}{\percent}} &  & {\SI{74.93}{\percent}} &  & {\SI{23.60}{\mega\nothing}} &  \\
MobilenetV2\textsuperscript{54} \cite{Sandler2018} & {} &  & {\SI{74.70}{\percent}} &  & {\SI{6.90}{\mega\nothing}} & \SI{0.58}{\nothing} \\
MobileNetV2(alpha=1.3)\textsuperscript{55} \cite{Sandler2018} & {\SI{92.12}{\percent}} &  & {\SI{74.68}{\percent}} &  & {\SI{3.80}{\mega\nothing}} &  \\
SFP\textsuperscript{56} \cite{he2018soft} & {\SI{92.06}{\percent}} &  & {\SI{74.61}{\percent}} &  & {} & \SI{2.19}{\nothing} \\
SSS-ResNetXt-35-A\textsuperscript{57} \cite{huang2018data} & {\SI{92.17}{\percent}} &  & {\SI{74.57}{\percent}} &  & {\SI{10.00}{\mega\nothing}} & \SI{2.07}{\nothing} \\
C-SGD-50\textsuperscript{58} \cite{ding2019centripetal} & {\SI{92.09}{\percent}} &  & {\SI{74.54}{\percent}} &  & {} & \SI{1.71}{\nothing} \\
Taylor-FO-BN-72\%\textsuperscript{59} \cite{Molchanov_2019_CVPR} & {} &  & {\SI{74.50}{\percent}} &  & {\SI{14.20}{\mega\nothing}} & \SI{2.25}{\nothing} \\
LFPC\textsuperscript{60} \cite{He_2020_CVPR} & {\SI{92.04}{\percent}} &  & {\SI{74.46}{\percent}} &  & {} & \SI{1.60}{\nothing} \\
NASNetMobile\textsuperscript{61} \cite{zoph2018learning} & {\SI{91.85}{\percent}} &  & {\SI{74.37}{\percent}} &  & {\SI{4.30}{\mega\nothing}} &  \\
SSS-ResNet-32\textsuperscript{62} \cite{huang2018data} & {\SI{91.91}{\percent}} &  & {\SI{74.18}{\percent}} &  & {\SI{18.60}{\mega\nothing}} & \SI{2.82}{\nothing} \\
GFP ResNet-50 3\textsuperscript{63} \cite{liu21ab} & {} &  & {\SI{73.94}{\percent}} &  & {} & \SI{1.02}{\nothing} \\
Pruned-90\textsuperscript{64} \cite{Luo2017b} & {\SI{91.60}{\percent}} &  & {\SI{73.56}{\percent}} &  & {\SI{23.89}{\mega\nothing}} & \SI{3.58}{\nothing} \\
MetaPruning 0.5 ResNet-50\textsuperscript{65} \cite{liu2019metapruning} & {} &  & {\SI{73.40}{\percent}} &  & {} & \SI{1.00}{\nothing} \\
ShaResNet-34\textsuperscript{66} \cite{Boulch2018} & {\SI{90.58}{\percent}} &  & {\SI{73.27}{\percent}} &  & {\SI{13.60}{\mega\nothing}} &  \\
SSS-ResNetXt-35-B\textsuperscript{67} \cite{huang2018data} & {\SI{91.58}{\percent}} &  & {\SI{73.17}{\percent}} &  & {\SI{8.50}{\mega\nothing}} & \SI{1.55}{\nothing} \\
Pruned-75\textsuperscript{68} \cite{Luo2017b} & {\SI{91.27}{\percent}} &  & {\SI{72.89}{\percent}} &  & {\SI{21.47}{\mega\nothing}} & \SI{3.19}{\nothing} \\
NISP-50-A\textsuperscript{69} \cite{Yu2018} & {} &  & {\SI{72.75}{\percent}} &  & {\SI{18.63}{\mega\nothing}} & \SI{2.76}{\nothing} \\
GDP 0.7\textsuperscript{70} \cite{lin2018accelerating} & {\SI{91.05}{\percent}} &  & {\SI{72.61}{\percent}} &  & {} & \SI{2.24}{\nothing} \\
ResNet-34-pruned-A\textsuperscript{71} \cite{Li2016} & {} &  & {\SI{72.56}{\percent}} &  & {\SI{19.90}{\mega\nothing}} & \SI{3.08}{\nothing} \\
ResNet-34-pruned-C\textsuperscript{72} \cite{Li2016} & {} &  & {\SI{72.48}{\percent}} &  & {\SI{20.10}{\mega\nothing}} & \SI{3.37}{\nothing} \\
NISP-34-A\textsuperscript{73} \cite{Yu2018} & {} &  & {\SI{72.29}{\percent}} &  & {\SI{15.74}{\mega\nothing}} & \SI{2.62}{\nothing} \\
ResNet-50 SSR-L2,0 A\textsuperscript{74} \cite{lin2019toward} & {\SI{91.73}{\percent}} &  & {\SI{72.29}{\percent}} &  & {\SI{15.50}{\mega\nothing}} & \SI{1.90}{\nothing} \\
ResNet-34-pruned-B\textsuperscript{75} \cite{Li2016} & {} &  & {\SI{72.17}{\percent}} &  & {\SI{19.30}{\mega\nothing}} & \SI{2.76}{\nothing} \\
ResNet-50 SSR-L2,1 A\textsuperscript{76} \cite{lin2019toward} & {\SI{91.57}{\percent}} &  & {\SI{72.13}{\percent}} &  & {\SI{15.90}{\mega\nothing}} & \SI{1.90}{\nothing} \\
NISP-50-B\textsuperscript{77} \cite{Yu2018} & {} &  & {\SI{72.07}{\percent}} &  & {\SI{14.36}{\mega\nothing}} & \SI{2.13}{\nothing} \\
ThiNet-70\textsuperscript{78} \cite{Luo2017} & {\SI{90.67}{\percent}} &  & {\SI{72.04}{\percent}} &  & {\SI{16.94}{\mega\nothing}} & \SI{4.88}{\nothing} \\
ResNet-50 HRank 2\textsuperscript{79} \cite{lin2020hrank} & {\SI{91.01}{\percent}} &  & {\SI{71.98}{\percent}} &  & {\SI{13.77}{\mega\nothing}} & \SI{1.55}{\nothing} \\
GDP 0.6\textsuperscript{80} \cite{lin2018accelerating} & {\SI{90.71}{\percent}} &  & {\SI{71.89}{\percent}} &  & {} & \SI{1.88}{\nothing} \\
SSS-ResNet-26\textsuperscript{81} \cite{huang2018data} & {\SI{90.79}{\percent}} &  & {\SI{71.82}{\percent}} &  & {\SI{15.60}{\mega\nothing}} & \SI{2.33}{\nothing} \\
Taylor-FO-BN-56\%\textsuperscript{82} \cite{Molchanov_2019_CVPR} & {} &  & {\SI{71.69}{\percent}} &  & {\SI{7.90}{\mega\nothing}} & \SI{1.34}{\nothing} \\
NISP-34-B\textsuperscript{83} \cite{Yu2018} & {} &  & {\SI{71.65}{\percent}} &  & {\SI{12.17}{\mega\nothing}} & \SI{2.02}{\nothing} \\
ResNet-50 SSR-L2,0 B\textsuperscript{84} \cite{lin2019toward} & {\SI{91.19}{\percent}} &  & {\SI{71.47}{\percent}} &  & {\SI{12.00}{\mega\nothing}} & \SI{1.70}{\nothing} \\
MobileNetV2(alpha=1.0)\textsuperscript{85} \cite{Sandler2018} & {\SI{90.14}{\percent}} &  & {\SI{71.34}{\percent}} &  & {\SI{2.30}{\mega\nothing}} &  \\
ResNet-50 SSR-L2,1 B\textsuperscript{86} \cite{lin2019toward} & {\SI{91.29}{\percent}} &  & {\SI{71.15}{\percent}} &  & {\SI{12.20}{\mega\nothing}} & \SI{1.70}{\nothing} \\
ThiNet-50\textsuperscript{87} \cite{Luo2017} & {\SI{90.02}{\percent}} &  & {\SI{71.01}{\percent}} &  & {\SI{12.38}{\mega\nothing}} & \SI{3.41}{\nothing} \\
GDP 0.5\textsuperscript{88} \cite{lin2018accelerating} & {\SI{90.14}{\percent}} &  & {\SI{70.93}{\percent}} &  & {} & \SI{1.57}{\nothing} \\
Pruned-50\textsuperscript{89} \cite{Luo2017b} & {\SI{90.03}{\percent}} &  & {\SI{70.84}{\percent}} &  & {\SI{17.38}{\mega\nothing}} & \SI{2.52}{\nothing} \\
MobileNet(alpha=1.0)\textsuperscript{90} \cite{howard2017mobilenets} & {\SI{89.50}{\percent}} &  & {\SI{70.42}{\percent}} &  & {\SI{3.20}{\mega\nothing}} &  \\
MobileNetV2(alpha=0.75)\textsuperscript{91} \cite{Sandler2018} & {\SI{89.18}{\percent}} &  & {\SI{69.53}{\percent}} &  & {\SI{1.40}{\mega\nothing}} &  \\
ResNet-50 GAL-1-joint\textsuperscript{92} \cite{lin2019towards} & {\SI{89.12}{\percent}} &  & {\SI{69.31}{\percent}} &  & {\SI{10.21}{\mega\nothing}} & \SI{1.11}{\nothing} \\
ResNet-50 HRank 3\textsuperscript{93} \cite{lin2020hrank} & {\SI{89.58}{\percent}} &  & {\SI{69.10}{\percent}} &  & {\SI{8.27}{\mega\nothing}} & \SI{0.98}{\nothing} \\
GreBdec (VGG-16)\textsuperscript{94} \cite{Yu2017} & {\SI{89.06}{\percent}} &  & {\SI{68.75}{\percent}} &  & {\SI{9.70}{\mega\nothing}} &  \\
ThiNet-30\textsuperscript{95} \cite{Luo2017} & {\SI{88.30}{\percent}} &  & {\SI{68.42}{\percent}} &  & {\SI{8.66}{\mega\nothing}} & \SI{2.20}{\nothing} \\
MobileNet(alpha=0.75)\textsuperscript{96} \cite{howard2017mobilenets} & {\SI{88.24}{\percent}} &  & {\SI{68.41}{\percent}} &  & {\SI{1.80}{\mega\nothing}} &  \\
GreBdec (GoogLeNet)\textsuperscript{97} \cite{Yu2017} & {\SI{88.11}{\percent}} &  & {\SI{68.30}{\percent}} &  & {\SI{1.50}{\mega\nothing}} &  \\
\end{tabular}
}
\end{table*}
\restoregeometry

\FloatBarrier


\bibliographystyle{plainnat}
\bibliography{library}